\documentclass[11pt]{article}

\usepackage{amsmath}
\usepackage{amssymb}
\usepackage{latexsym}
\usepackage[round]{natbib}
\usepackage{makeidx}
\usepackage{nomencl}
\usepackage{graphicx}
\usepackage[arrow, matrix, curve]{xy}
\usepackage{mathrsfs}
\usepackage{dsfont}   
\usepackage{hyperref}

 \newtheorem{theorem}{Theorem}[section]
 
 \newtheorem{lemma}[theorem]{Lemma}
 \newtheorem{remark}[theorem]{Remark}
 \newtheorem{example}[theorem]{Example}
 
 \newtheorem{proposition}[theorem]{Proposition}

 \newenvironment{proof}
   {\begin{list}{\textbf{Proof}:}
                {\setlength{\leftmargin}{0em}
                 \setlength{\labelwidth}{-0.5em}
                }
   }
   {\hspace*{\fill}$\Box$\end{list}}
   
 \newenvironment{proof*}
   {\begin{list}{\textbf{Proof}:}
                {\setlength{\leftmargin}{0em}
                 \setlength{\labelwidth}{-0.5em}
                }
   }
   {\end{list}}

  \newcommand{\R}{\mathds{R}}
  \newcommand{\N}{\mathds{N}}

\newcommand{\E}{\mathds{E}}

\newcommand{\la}{\langle}
\newcommand{\ra}{\rangle_{H}}

\newcommand{\SVMlpl}{f_{L,P,\lambda}}
\newcommand{\SVMlplz}{f_{L,P,\lambda_{0}}}

\newcommand{\SVMldldn}{f_{L,\mathbf{D}_{n},\lambda_{\mathbf{D}_{n}}}}
\newcommand{\SVMldoldon}{%
   f_{L,\mathbf{D}_{n}(\omega),\lambda_{\mathbf{D}_{n}(\omega)}}}

\newcommand{\SVMldlz}{f_{L,\mathbf{D}_{n},\lambda_{0}}}

\newcommand{\SVMif}{f_{\iota(F)}}
\newcommand{\SVMifn}{f_{\iota(F_{n})}}
\newcommand{\SVMifz}{f_{\iota(F_{0})}}

\newcommand{\Ls}{L^{\prime}}
\newcommand{\Lss}{L^{\prime\prime}}

\newcommand{\X}{\mathcal{X}}
\newcommand{\Y}{\mathcal{Y}}
\newcommand{\Z}{\mathcal{Z}}
\newcommand{\XX}{\X\times\X}
\newcommand{\XY}{\X\times\Y}

\newcommand{\G}{\mathcal{G}}
\newcommand{\F}{\mathcal{F}}
\newcommand{\Rs}{\mathcal{R}}

\begin{document}

\title{Asymptotic Normality of   
       Support Vector Machine Variants
       and Other Regularized Kernel Methods}

 \author{Robert Hable \\ 
         Department of Mathematics \\
         University of Bayreuth
         }
 \date{}
         
\maketitle

\begin{abstract}
  In nonparametric classification and regression problems,
  regularized kernel methods, in particular
  support vector machines, 
  attract much attention
  in theoretical and in applied statistics.
  In an abstract sense, 
  regularized kernel methods
  (simply called SVMs here)
  can be seen as 
  regularized M-estimators
  for a parameter
  in a (typically infinite dimensional) reproducing kernel Hilbert space. 
  For smooth loss functions $L$, 
  it is shown that
  the difference between the estimator, i.e.\ the
  empirical SVM $\SVMldldn$,   
  and the theoretical SVM $\SVMlplz$
  is asymptotically normal with rate $\sqrt{n}$.
  That is, $\sqrt{n}(\SVMldldn-\SVMlplz)$ converges weakly to a
  Gaussian process in the reproducing kernel Hilbert space.
  As common in real applications, the choice of the
  regularization parameter $\mathbf{D}_n$ 
  in $\SVMldldn$ may depend on the data.   
  The proof is done by an application of the functional delta-method
  and by showing that the SVM-functional 
  $P\mapsto\SVMlpl$ is
  suitably Hadamard-differentiable.
\end{abstract}


\section{Introduction}  \label{sec-introduction}

One of the most important tasks in statistics is
the estimation of the influence of an input variable $X$
on an output variable $Y$. 
On the basis of a finite data set
$(x_{1},y_{1}),\dots,(x_{n},y_{n})\in
 \mathcal{X}\times\mathcal{Y}$\,, 
the goal is to find an ``optimal'' predictor 
$f:\mathcal{X}\rightarrow\mathcal{Y}$ which
makes a prediction $f(x)$ for an
unobserved $y$\,. In case of a finite space $\Y$, this is called 
classification and,
in case of an infinite space $\Y\subset\R$, this is called 
regression. Often,
a signal plus noise relationship 
$y=f_{0}(x)+\varepsilon$ 
is assumed and the task is to estimate the unknown 
regression function $f_{0}$\,. 
In parametric statistics, it is assumed that
$f_{0}$ is contained in a known finite-dimensional function
space. This assumption is dropped or, at least, 
considerably weakened in nonparametric statistics.
In nonparametric classification and
regression problems, regularized kernel methods, in particular
support vector machines,
recently attract much attention
in theoretical and in applied statistics;
see e.g.\ the comprehensive books \cite{vapnik1998}, \cite{schoelkopf2002},
and \cite{steinwart2008} and the references cited therein.
For convenience, a large class of 
regularized kernel methods
for classification and regression
(based on any loss function)
is called ``support vector machine'' (SVM) in 
the following, e.g.\ as in
\cite{steinwart2008}.
That is, the 
term ``support vector machine'' (SVM) is used in a
broad sense here whereas,
originally,
the term ``support vector machine'' was coined for the special
case where $\Y=\{-1,1\}$ (binary classification) and
where the loss function $L$ is the so-called \emph{hinge-loss}.

Typically, the weaker assumptions 
in nonparametric statistics
have to be compensated by an increase of observations
in order to obtain the same precision of the estimation. 
Nevertheless, it is well-known that some
nonparametric estimators still are asymptotically normal
for the same rate $\sqrt{n}$ as many parametric estimators. 
In this article, it is shown that 
also support vector machines based on smooth 
loss functions
enjoy an asymptotic normality property
for the rate $\sqrt{n}$.
For an i.i.d.\ sample $D_{n}=\big((x_{1},y_{1}),\dots,(x_{1},y_{n})\big)$
from a distribution $P$,
the \textit{empirical SVM} is a
function $f_{L,D_{n},\lambda_{D_n}}$ which solves the minimization problem
\begin{eqnarray}\label{introduction-empirical-svm}
  \min_{f\in H}\,
  \frac{1}{n}\sum_{i=1}^{n} L\big((x_{i},y_{i},f(x_{i})\big)
  \,+\,\lambda_{D_n}\|f\|_{H}^{2}\;,
\end{eqnarray}
where $L$ is a loss function and
$H$ is a certain space of functions $f:\X\rightarrow\R$, namely a
so-called \emph{reproducing kernel Hilbert space}.
The first term in
(\ref{introduction-empirical-svm}) is the empirical mean of
the losses caused by the predictions $f(x_{i})$ and
the second term penalizes the complexity of $f$ in order to
avoid overfitting; the regularization parameter $\lambda_{D_n}$ 
is a positive real number which is typically chosen in a 
data-driven way, e.g., by cross-validation. 

Depending on the size of the space $H$, SVMs can be used
as a parametric or a non-parametric method. Choosing
a finite-dimensional $H$ leads to a parametric setting,
choosing an infinite-dimensional $H$ leads to
a non-parametric setting. In the 
parametric setting, asymptotic normality of support vector machines
in the original sense (binary classification using the 
hinge loss) has already been investigated:   
\cite{jiang2008} derive asymptotic normality of the 
estimated prediction error of SVMs with finite-dimensional $H$. 
Under some regularity conditions on the distribution of the
data, \cite{koo2008} show asymptotic normality of the
coefficients of the linear SVM 
(i.e., $H$ only contains linear functions).
In the following, a general non-parametric setting (covering
classification and regression) is considered but, 
by going over from parametrics to non-parametrics, we have to impose
a bound on the complexity of the predictor. 
Instead of estimating a solution 
$f_{L,P}^{\ast}$
of the (ill-posed) minimization problem
\begin{eqnarray}\label{introduction-theoretical-svm-pre}
  \min_{f\in H}\,
  \int L\big((x,y,f(x)\big)\,P\big(d(x,y)\big)\;,
\end{eqnarray}
we estimate a smoother approximation, namely the solution
$\SVMlplz$
of the minimization problem
\begin{eqnarray}\label{introduction-theoretical-svm}
  \min_{f\in H}\,
  \int L\big((x,y,f(x)\big)\,P\big(d(x,y)\big)
  \,+\,\lambda_{0}\|f\|_{H}^{2}
\end{eqnarray}
for a fixed regularization parameter $\lambda_0\in(0,\infty)$.
The minimizer $\SVMlplz$ of (\ref{introduction-theoretical-svm})
is called \emph{theoretical SVM}.
This so-called Tikhonov regularization is equivalent to a
minimization problem
$$\int L\big((x,y,f(x)\big)\,P\big(d(x,y)\big)
  \;\;=\;\;\min!
  \qquad\quad f\in H,\quad\|f\|_{H}\leq r_{0}
$$
where $r_{0}$ can be interpreted as an upper bound on the
complexity of the function $f$; a smaller $\lambda_0>0$ corresponds
to a larger $r_0>0$.
It will be shown that the sequence of 
SVM-estimators
$$(\XY)^n\;\rightarrow\;H\,,\qquad
  D_{n}\;\mapsto\;f_{L,D_{n},\lambda_{D_n}}
$$
is asymptotically normal for the rate $\sqrt{n}$ if
the empirical SVM $f_{L,D_{n},\lambda_{D_n}}$ is shifted
by the theoretical SVM $\SVMlplz$.
That is,
$$\sqrt{n}\big(f_{L,D_{n},\lambda_{D_n}}-\SVMlplz\big)
$$
converges weakly to a (zero-mean) Gaussian process in 
the function space $H$. This also
implies asymptotic normality of the risk 
$$\sqrt{n}
  \Big(
    \mathcal{R}_{L,P}\big(f_{L,D_{n},\lambda_{D_n}}\big)
    -\mathcal{R}_{L,P}\big(\SVMlplz\big)
  \Big)
  \;\;\leadsto\;\;\sigma\mathcal{N}(0,1)
$$
where 
$\mathcal{R}_{L,P}(f)=\int L(x,y,f(x))\,P\big(d(x,y)\big)$
denotes the risk of a predictor $f$ and
$\sigma\in[0,\infty)$. 
The regularization parameter $\lambda_{D_n}$ 
for the empirical SVM may depend on the data. We only need 
that $\sqrt{n}(\lambda_{D_n}-\lambda_0)$ converges to 0 in probability.
This will be proven by an advanced application of a
functional delta-method. Accordingly, it will be shown that
the map $P\mapsto\SVMlpl$ is suitably Hadamard-differentiable. 
According to (\ref{introduction-empirical-svm}) and
(\ref{introduction-theoretical-svm}), SVMs can be 
seen as (regularized) M-estimators
for a parameter
in a typically infinite dimensional Hilbert space. 
Asymptotic normality of M-estimators
for finite-dimensional parameters
and rates of convergence of M-estimators
for parameters
in metric spaces are considered in \cite{vandegeer2000}.

\smallskip

Of course, it would be desirable to dispense with the
complexity bound and to have asymptotic normality of
$$\sqrt{n}(f_{L,D_{n},\lambda_{D_n}}-f_{L,P}^{\ast}) 
  \qquad\text{instead of}\qquad
  \sqrt{n}(f_{L,D_{n},\lambda_{D_n}}-f_{L,P,\lambda_0})
$$
-- \,if $f_{L,P}^{\ast}$ exists at all. However, in the 
non-parametric setting where $H$ is a large infinite-dimensional
function space, this is not possible. Such a result would violate
the no-free-lunch theorem which, roughly speaking, yields that
there is no uniform rate of convergence
without such a bound on the complexity.
It is only possible to get uniform rates of convergence
within special classes of distributions.
The investigation of
rates of convergence 
for special cases -- e.g. classification
under assumptions on the unknown true probability measure such as
Tsybakov's noise assumption 
\cite[p.\ 138]{Tsybakov2004}
-- is one of the most
important topics of recent research about support vector machines
and related learning methods; see e.g.\ 
\cite{steinwartscovel2007}, \cite{caponnetto2007},
\cite{blanchard2008}, \cite{steinwarthush2009},
\cite{mendelson2010}. It is a matter of further research if
similar assumptions on the unknown true probability measures
allow asymptotic normality of
$\sqrt{n}(f_{L,D_{n},\lambda_{D_n}}-f_{L,P}^{\ast})$.

\smallskip

The article is organized as follows: 
Section \ref{subsec-setup} briefly recalls the definition
of support vector machines in a broad sense and fixes the
notation. Section \ref{sec-main-results} contains the main
results concerning asymptotic normality of support vector machines
and their risks. Since the proof is quite involved, it is deferred
to the appendix but
Section \ref{sec-sketch-of-proof} provides a short outline.
Finally, Sections \ref{sec-conclusions} contains some
concluding remarks.

\section{Support Vector Machines}\label{subsec-setup}

Let $(\Omega,{\cal A},Q)$ be a probability space, 
let $\mathcal{X}$ be a closed and bounded subset of
$\R^d$,
and let $\mathcal{Y}$ be a 
closed subset of $\mathds{R}$ with Borel-$\sigma$-algebra
$\mathfrak{B}(\mathcal{Y})$\,. The Borel-$\sigma$-algebra
of $\mathcal{X}\times\mathcal{Y}$ is denoted by
$\mathfrak{B}(\mathcal{X}\times\mathcal{Y})$.
Let
\begin{eqnarray*}
  X_{1},\dots,X_{n}\;:\;\;(\Omega,{\cal A},Q)
  \;\longrightarrow\;\big(\mathcal{X},\mathfrak{B}(\mathcal{X})\big)\,, \\
  Y_{1},\dots,Y_{n}\;:\;\;(\Omega,{\cal A},Q)
  \;\longrightarrow\;\big(\mathcal{Y},\mathfrak{B}(\mathcal{Y})\big)\;\;
\end{eqnarray*}
be random variables such that
$\,(X_{1},Y_{1}),\dots,(X_{n},Y_{n})\,$ are independent 
and identically distributed according to some unknown
probability measure
$P$ on
$\big(\mathcal{X}\times\mathcal{Y},
  \mathfrak{B}(\mathcal{X}\times\mathcal{Y})
 \big)
$.
Define
$$\mathbf{D}_{n}\;:=\;\big((X_{1},Y_{1}),\dots,(X_{n},Y_{n})\big)
  \qquad\forall\,n\in\mathds{N}\;.
$$

A measurable map
$\,L:\mathcal{X}\times\mathcal{Y}\times\mathds{R}\rightarrow[0,\infty)\,
$
is called \emph{loss function}. A loss function $L$ is called
\emph{convex} loss function if it is convex in its third argument, 
i.e.
$t\mapsto L(x,y,t)$ is convex for every $(x,y)\in\XY$. Furtheremore,
a loss function $L$ is called $P$-integrable Nemitski loss function
of order $p\in[1,\infty)$ if there is a $P$-integrable function
$b:\XY\rightarrow\R$ such that
$$\big|L(x,y,t)\big|\;\leq\;b(x,y)+|t|^{p}
 \qquad\forall\,(x,y,t)\in\XY\times\R\;.
$$
If $b$ is even $P$-\emph{square}-integrable, $L$ is called
$P$-\emph{square}-integrable Nemitski loss function
of order $p\in[1,\infty)$.
The \emph{risk} of a measurable function 
$f:\mathcal{X}\rightarrow\mathds{R}$
is defined by
$$\mathcal{R}_{L,P}(f)\;=\;
  \int_{\mathcal{X}\times\mathcal{Y}}L\big(x,y,f(x)\big)\,
  P\big(d(x,y)\big)\;.
$$
The goal is to estimate a function $f:\mathcal{X}\rightarrow\R$
which minimizes this risk. 
The estimates obtained from the method of support vector machines
are elements of so-called reproducing kernel Hilbert spaces
(RKHS) $H$. A RKHS $H$ is a certain Hilbert space 
of functions $f:\mathcal{X}\rightarrow\mathds{R}$ which is
generated by a \emph{kernel}  
$k:\mathcal{X}\times\mathcal{X}\rightarrow\mathds{R}$\,.
See e.g.\ \cite{schoelkopf2002} or \cite{steinwart2008}
for details about these concepts.

Let $H$ be such a RKHS. Then, the
\emph{regularized risk} of an element $f\in H$ is defined to be
$$\mathcal{R}_{L,P,\lambda}(f)
  \;=\;\mathcal{R}_{L,P}(f)\,+\,\lambda\|f\|_{H}^2\;,
  \qquad\text{where}\;\;\;\lambda\in(0,\infty)\,.
$$

An element $f\in H$ is called a \emph{support vector machine} 
and denoted by $f_{L,P,\lambda}$ if it
minimizes the regularized risk in $H$\,. That is,
$$\mathcal{R}_{L,P}(f_{L,P,\lambda})
      \,+\,\lambda\|f_{L,P,\lambda}\|_{H}^2\;=\;
  \inf_{f\in H}\,\mathcal{R}_{L,P}(f)\,+\,\lambda\|f\|_{H}^2\;.
$$
The \textit{SVM-estimator} is defined by
$$S_{n}\;:\;\;(\mathcal{X}\times\mathcal{Y})^n\;\rightarrow\;H\,,\qquad
  D_{n}\;\mapsto\;
  f_{L,D_{n},\lambda_{D_n}}
$$
where $f_{L,D_{n},\lambda_{D_n}}$ is that function $f\in H$ which minimizes
\begin{eqnarray}\label{setup-regularized-empirical-risk}
  \frac{1}{n}\sum_{i=1}^{n}L\big(x_{i},y_{i},f(x_{i})\big)
  \,+\,\lambda_{D_n}\|f\|_{H}^{2}
\end{eqnarray}
in $H$ for 
$D_{n}=((x_{1},x_{2}),\dots,(x_{n},y_{n}))\,\in\,
 (\mathcal{X}\times\mathcal{Y})^{n}
$\,.
The empirical support vector machine $f_{L,D_{n},\lambda_{D_n}}$ 
uniquely exists for every $\lambda_{D_n}\in(0,\infty)$ and every
data-set $D_{n}\in(\XY)^{n}$
if $t\mapsto L(x,y,t)$ is convex for every $(x,y)\in\XY$.

\smallskip

The symbol $\leadsto$ denotes
weak convergence of probability measures
or random variables.

\section{Asymptotic Normality}\label{sec-asymptotic-normality}

\subsection{Main Results}\label{sec-main-results}

The following theorems provide the main results.
For random sequences of regularization parameters
$(\lambda_{\mathbf{D}_n})_{n\in\N}\subset(0,\infty)$ which converges
in probability with rate $\sqrt{n}$ to some
$\lambda_{0}\in(0,\infty)$\,,
Theorem \ref{theorem-sqrt-n-consistency} 
says that
the $\sqrt{n}$-standardized difference between the 
empirical support vector machine $\SVMldldn$
and the theoretical support vector machine
$\SVMlplz$ is asymptotically normal
under some relatively mild conditions.
That is, the $H$-valued 
random variable
$$\Omega\;\rightarrow\;H\,,\qquad
  \omega\;\rightarrow\;\sqrt{n}(\SVMldoldon-\SVMlplz)
$$
converges weakly to a random variable
$$\mathds{H}\;:\;\;\Omega\;\rightarrow\;H\,,\quad\;
    \omega\;\mapsto\;\mathds{H}(\omega)
$$ 
which is a Gaussian process in $H$\,. 
Accordingly, for every finite collection of functions
$\{f_{1},\dots,f_{m}\}\subset H$, the 
random variable
$$\Omega\;\rightarrow\;\R^m\,,\qquad
  \omega\;\mapsto\;
  \Big(\big\la f_{1},\mathds{H}(\omega) \big\ra,\dots,
       \big\la f_{m},\mathds{H}(\omega) \big\ra
  \Big)
$$
has a multivariate normal distribution.
In particular, the reproducing property of $k$ implies that,
for every $x_{1},\dots,x_{m}\,\in\,\mathcal{X}$\,,
$$\sqrt{n}
  \left(
     \begin{array}{c}
       \SVMldldn(x_{1})-\SVMlplz(x_{1}) \\
       \vdots \\
       \SVMldldn(x_{m})-\SVMlplz(x_{m}) \\
     \end{array}
  \right)
  \;\;\leadsto\;\;\mathcal{N}_{m}(0,\Sigma)
$$
where $\Sigma$ is a covariance matrix. In addition,
Theorem \ref{theorem-sqrt-n-consistency-risks}
provides $\sqrt{n}$-consis\-tency of the risk.

\begin{theorem}\label{theorem-sqrt-n-consistency}
  Let $\mathcal{X}\subset\R^d$ be closed and bounded and
  let $\mathcal{Y}\subset\R$ be closed. 
  Assume that $k:\XX\rightarrow\R$ is the restriction of an
  $m$\,-\,times continuously differentiable kernel 
  $\tilde{k}:\R^d\times\R^d\rightarrow\R$
  such that
  $m>d/2$ and $k\not=0$.
  Let $H$ be the RKHS of $k$
  and let $P$ be a probability measure on $(\XY,\mathfrak{B}(\XY))$\,.
  Let 
  $$L\;:\;\;\XY\times\R\;\rightarrow\;[0,\infty)\,,\qquad
    (x,y,t)\;\mapsto\;L(x,y,t)
  $$
  be a convex, $P$-square-integrable Nemitski loss function 
  of order $p\in[1,\infty)$
  such that
  the partial derivatives
  $$\Ls(x,y,t)\;:=\;\frac{\partial L}{\partial t}(x,y,t)
    \qquad\text{and}\qquad
    \Lss(x,y,t)\;:=\;\frac{\partial^2 L}{\partial^2 t}(x,y,t)
  $$
  exist for every $(x,y,t)\in\XY\times\R$\,.
  Assume that the maps
  $$(x,y,t)\;\mapsto\;\Ls(x,y,t) \qquad\text{and}\qquad
    (x,y,t)\;\mapsto\;\Lss(x,y,t)
  $$
  are continuous. Furthermore, assume that for every $a\in(0,\infty)$,
  there is a $b_{a}^{\prime}\in L_{2}(P)$ and a constant
  $b_{a}^{\prime\prime}\in [0,\infty)$ such that,
  for every $(x,y)\in\XY$,
  \begin{eqnarray}\label{theorem-sqrt-n-consistency-1}
    \sup_{t\in[-a,a]}\big|\Ls(x,y,t)\big|\;\leq\;b_{a}^{\prime}(x,y)
    \quad\;\text{and}\quad\;
    \sup_{t\in[-a,a]}\big|\Lss(x,y,t)\big|\;\leq\;
         b_{a}^{\prime\prime}\;.
  \end{eqnarray}
  Then, for every $\lambda_{0}\in(0,\infty)$,
  there is a tight, Borel-measurable
  Gaussian process 
  $$\mathds{H}\;:\;\;\Omega\;\rightarrow\;H\,,\quad\;
    \omega\;\mapsto\;\mathds{H}(\omega)
  $$ 
  such that, 
  \begin{eqnarray}\label{theorem-sqrt-n-consistency-2}
    \sqrt{n}\big(f_{L,\mathbf{D}_{n},\lambda_{\mathbf{D}_{n}}}-\SVMlplz\big)
    \;\;\leadsto\;\;\mathds{H}
    \qquad \text{in}\;\;H
  \end{eqnarray}
  for every Borel-measurable sequence
  of random regularization parameters $\lambda_{\mathbf{D}_n}$ with
  $$\sqrt{n}\big(\lambda_{\mathbf{D}_n}-\lambda_0\big)
    \;\xrightarrow[\;n\rightarrow\infty\;]{}\;0
    \qquad\text{in probability\,.}
  $$
  The Gaussian process $\mathds{H}$ 
  is zero-mean; i.e.,
  $\mathds{E}\la f,\mathds{H}\ra=0$ for every $f\in H$\,.
\end{theorem}

By use of tis theorem, the following asymptotic
result on the risks is obtained.

\begin{theorem}\label{theorem-sqrt-n-consistency-risks}
  Under the assumptions of Theorem 
  \ref{theorem-sqrt-n-consistency},
  there is, for every $\lambda_{0}\in(0,\infty)$,
  a constant $\sigma\in[0,\infty)$ such that
  $$\sqrt{n}
    \big(\mathcal{R}_{L,P}(f_{L,\mathbf{D}_{n},\lambda_{\mathbf{D}_{n}}})
                 -\mathcal{R}_{L,P}(\SVMlplz)
    \big)
    \;\;\leadsto\;\;\sigma\mathcal{N}(0,1)
  $$
  for every Borel-measurable sequence
  of random regularization parameters $\lambda_{\mathbf{D}_n}$ with
  $\sqrt{n}\big(\lambda_{\mathbf{D}_n}-\lambda_0\big)
    \xrightarrow[\;n\rightarrow\infty\;]{}0$
  in probability.
\end{theorem}

\smallskip

  According to the above theorems, 
  the Gaussian process $\mathds{H}$ 
  and the constant $\sigma$ do not depend on the sequence
  $\lambda_{\mathbf{D}_n}$, $n\in\N$, 
  but only on $\lambda_0$.  
  Though it is possible that $\mathds{H}$ 
  degenerates to 0,
  this only happens 
  in trivial cases, e.g., if 
  $P$ is equal to a Dirac distribution,
  or $|Y|\leq\varepsilon$ while using a smoothed version of the
  epsilon-insensitive loss; see Remark
  \ref{remark-degenerated-limit-distribution}.
  If  
  the constant $\sigma$ is equal to 0
  in Theorem \ref{theorem-sqrt-n-consistency-risks},
  the limit degenerates to 0. In contrast to
  $\mathds{H}$, this not only happens in degenerated
  cases. For example, it is known that
  the rate of convergence of the risk is faster
  than $\sqrt{n}$ in some cases 
  (see e.g.\ \cite{steinwartscovel2007})
  which leads to a degenerated limit 
  in Theorem \ref{theorem-sqrt-n-consistency-risks}.

\smallskip

As stated above, the results are true under 
some relatively mild assumptions. 
In particular, the assumptions on
$k$ are fulfilled for all of the most common kernels
(e.g.\ Gaussian RBF kernel, polynomial kernel, exponential kernel, 
linear kernel). It is assumed that
the loss function is two times continuously 
differentiable in the
third argument. On the one hand, this is
an obvious restriction because some of the most
common loss functions are not differentiable:
the epsilon-insensitive loss for regression
and the hinge loss for classification.
On the other hand, this assumption is not based on
any unknown entity such as the model distribution $P$\,.
In particular, a practitioner can a priori
meet this requirement by a suitable choice of the
loss function; e.g.\ the least-squares loss
for regression and the logistic loss for classification.
This is contrary to the noise assumptions common
in order to establish rates of convergence to the Bayes risk
because such assumptions depend on the unknown $P$
so that they can hardly be checked in applications. 
In addition, Remark \ref{remark-diffble-epsilon-version-lossfunction} 
describes how a Lipschitz-continuous
loss function (such as the epsilon-insensitive loss and the hinge loss) 
can always be turned into a differentiable $\varepsilon$-version
of the loss function. That is, though
the theorem does not
cover support vector machines in the 
original terminology, it covers variants based on
a slightly smoothed hinge loss.

In order to ensure mere existence of the theoretical
SVM $\SVMlplz$\,, it is necessary 
to assume a $P$\,-\,integrabilty condition. For example,
it is common to assume that $L$ is a $P$\,-\,integrable
Nemitski loss function \cite{christmannsteinwart2007}. 
In order to obtain asymptotic normality in the above theorems, 
we assume
that $L$ is a $P$\,-\,\emph{square}-integrable
Nemitski loss function which seems to be a natural 
assumption in view of the square-integrability
assumptions for usual central limit theorems.
In addition, a similar $P$\,-\,integrabilty condition
is assumed for the derivative of the loss function.
If $\Y$ is bounded (as, e.g., in case of a classification problem)
and $L$, $\Ls$ and $\Lss$ are continuous,
all of the integrability assumptions are fulfilled.

In order to fulfill 
$$\sqrt{n}\big(\lambda_{\mathbf{D}_n}-\lambda_0\big)
    \;\xrightarrow[\;n\rightarrow\infty\;]{}\;0
    \qquad\text{in probability},
$$
(which is the only assumption on the random
sequence of regularization parameters), 
it is possible to use any data-driven method for choosing
the regularization parameter. The only thing one has to do is
to choose a (possibly large) constant $c\in(0,\infty)$ and
to make sure that the method (e.g.\ cross validation)
picks a value from $[\lambda_0\,,\,\lambda_0+c/\sqrt{n\ln(n)}\,]$.
Note that, as the notation suggests, it is 
indeed possible to
use the same data for choosing the regularization parameter
as for building the final SVM - just as usually done by practitioneers, 
e.g., when applying cross validation.

The following examples list 
some general situations in which Theorems 
 \ref{theorem-sqrt-n-consistency} and
 \ref{theorem-sqrt-n-consistency-risks} 
are applicable.

\begin{example}[Classification] 
 Theorems 
 \ref{theorem-sqrt-n-consistency} and
 \ref{theorem-sqrt-n-consistency-risks}
 are applicable in 
 the following setting for a classification problem:
 \begin{itemize}
  \item $\X$ bounded and closed, $\Y=\{-1;1\}$ 
  \item $k$ a Gaussian RBF kernel, 
    a polynomial kernel, an exponential kernel 
    or a linear kernel
  \item $L$ the least-squares loss or the logistic loss
 \end{itemize}
\end{example}
\begin{example}[Regression] 
  Theorems 
 \ref{theorem-sqrt-n-consistency} and
 \ref{theorem-sqrt-n-consistency-risks}
 are applicable in 
 the following setting for a regression problem:
 \begin{itemize}
  \item $\X$ bounded and closed, $\Y$ closed
  \item $k$ a Gaussian RBF kernel, 
    a polynomial kernel, an exponential kernel 
    or a linear kernel
  \item $L$ the least-squares loss
  \item $P$ such that $\int y^4 \,P\big(d(x,y)\big)\,<\,\infty$
 \end{itemize}
\end{example}

The following Remark \ref{remark-diffble-epsilon-version-lossfunction} 
describes how a Lipschitz-continuous
loss function 
can always be turned into a differentiable $\varepsilon$-version
of the loss function such that all of the assumptions on the 
partial derivatives
$\Ls$ and $\Lss$ are automatically fulfilled. In particular,
the proposed construction works for
the epsilon-insensitive loss and the hinge loss.
\begin{remark}[Smoothing loss functions by use
  of mollifiers]
\label{remark-diffble-epsilon-version-lossfunction}
  Let $L:\XY\times\R\rightarrow[0,\infty)$ be a convex
  $P$-square-integrable Nemitski loss function of order
  $p\in[1,\infty)$. Assume that $L$ is also a
  Lipschitz-continuous loss function. That is, there is
  a constant $b^{\prime}\in(0,\infty)$ such that
  $$\sup_{(x,y)\in\XY}\big|L(x,y,t_{1})-L(x,y,t_{2})\big|
    \;\leq\;b^{\prime}|t_{1}-t_{2}|
    \qquad\forall\,t_{1},t_{2}\in\R\;.
  $$
  Then, for every $\varepsilon>0$, it is possible to construct
  a loss function $L_{\varepsilon}$ such that
  \begin{eqnarray}\label{remark-diffble-epsilon-version-lossfunction-1}
    \big|L(x,y,t)-L_{\varepsilon}(x,y,t)\big|\;\leq\;\varepsilon
    \qquad\quad\forall\,(x,y,t)\in\XY\times\R
  \end{eqnarray}
  and all of 
  the assumptions of Theorems 
  \ref{theorem-sqrt-n-consistency} and
  \ref{theorem-sqrt-n-consistency-risks} are fulfilled for 
  $L_{\varepsilon}$. 

  This can be done in the following way:
  Take a so-called mollifier function $\varphi:\R\rightarrow\R$; e.g.,
  $$\varphi\;:\;\;\R\;\rightarrow\;\R\,,\qquad
    t\;\mapsto\;\gamma^{-1} e^{-\frac{1}{1-t^2}}I_{(-1,1)}(t)
  $$
  where $\gamma\in(0,\infty)$ is chosen so that
  $\int \varphi \,d\lambda=1$. 
  (See e.g.\ \cite[p.\ 341ff]{denkowski2003} for the concept of
  mollifiers and their basic properties.)
  Define 
  $\varphi_{\varepsilon}(s)=\varphi(sb^{\prime}/\varepsilon)$
  for every $s\in\R$ and 
  \begin{eqnarray}\label{remark-diffble-epsilon-version-lossfunction-101}
    L_{\varepsilon}(x,y,t)\,=
    \frac{b^{\prime}}{\varepsilon}
    \int \!\varphi_{\varepsilon}(s)
         L(x,y,t-s)\,
    \lambda(ds)         
    \qquad\forall\,(x,y,t)
  \end{eqnarray}
  Then, 
  (\ref{remark-diffble-epsilon-version-lossfunction-1})
  follows from an easy calculation using 
  Lipschitz-continuity of $L$. The $\varepsilon$-version
  $L_{\varepsilon}$ is again a convex
  $P$-square-integrable Nemitski loss function of order
  $p\in[1,\infty)$.
  For every $(x,y,t)\in\XY\times\R$, the function 
  $t\mapsto L_{\varepsilon}(x,y,t)$ is infinitely differentiable
  and the derivatives are given by 
  \begin{eqnarray}\label{remark-diffble-epsilon-version-lossfunction-2}
    \frac{\,\partial^m}{\partial^m t}\,L_{\varepsilon}(x,y,t)
    \;=\;
    \frac{b^{\prime}}{\varepsilon}
    \int \!\frac{\partial^m\varphi_{\varepsilon}}{\partial^m s}(s)
         L(x,y,t-s)\,
    \lambda(ds) \;.        
  \end{eqnarray}
  Furthermore, for every $(x,y,t)\in\XY\times\R$,
  \begin{eqnarray}
    \label{remark-diffble-epsilon-version-lossfunction-3}
    \big|\Ls_{\varepsilon}(x,y,t)\big|\;=\;
    \left|\frac{\,\partial}{\partial t}\,L_{\varepsilon}(x,y,t)\right|
    \;\leq\;
    b^{\prime}\;, \\
    \label{remark-diffble-epsilon-version-lossfunction-4}
    \big|\Lss_{\varepsilon}(x,y,t)\big|\,=
    \left|\frac{\,\partial^2}{\partial^2 t}\,L_{\varepsilon}(x,y,t)\right|
    \,\leq\,b^{\prime}\cdot
        \frac{b^{\prime}}{\varepsilon}
        \int \!\frac{\partial\varphi_{\varepsilon}}{\partial s}(s)
        \lambda(ds)
    \,=:\,b^{\prime\prime}\,.
  \end{eqnarray}
  Inequality (\ref{remark-diffble-epsilon-version-lossfunction-3})
  follows from 
  the definition of derivatives by means of
  difference quotients, 
  (\ref{remark-diffble-epsilon-version-lossfunction-101}), and 
  Lipschitz-continuity of $L$.
  Inequality (\ref{remark-diffble-epsilon-version-lossfunction-4})
  follows from the definition of derivatives by means of 
  difference quotients,
  (\ref{remark-diffble-epsilon-version-lossfunction-2}) for $m=1$,
  and Lipschitz-continuity of $L$.
  
  In particular, the construction of such an $\varepsilon$-version
  of $L$ works for the hinge loss (classification) and,
  if $\int y^2\,P(d(x,y))\,<\,\infty$, for the epsilon-insensitive 
  loss (regression). 
  Another approach in order to obtain smooth approximations of
  loss functions is proposed in \cite{dekelsinger2005}.
\end{remark}

The following Remark \ref{remark-degenerated-limit-distribution}
shows that the limit distribution 
in Theorem \ref{theorem-sqrt-n-consistency} is only 
degenerated in trivial cases.
\begin{remark}[Degenerated limit distribution]
  \label{remark-degenerated-limit-distribution}
  As shown in 
  Proposition \ref{prop-degenerated-limit} in the appendix,
  the Gaussian
  process $\mathds{H}$ in
  $$\sqrt{n}\big(f_{L,\mathbf{D}_{n},
               \lambda_{\mathbf{D}_{n}}}-\SVMlplz
          \big)
    \;\;\leadsto\;\;\mathds{H}
  $$
  (Theorem \ref{theorem-sqrt-n-consistency})   
  is degenerated to 0 if and only if,
  for every $h\in H$, there is 
  a constant $c_h\in\R$ such that
  \begin{eqnarray}
    \label{remark-degenerated-limit-distribution-1}
    L^\prime\big(x,y,\SVMlplz(x)\big)h(x)
    \;=\;c_h\quad \;
    \text{for}\quad
    P\,-\,\textup{a.e. }\,
    (x,y)\in\XY\;.\quad
  \end{eqnarray}
  This only happens in trivial cases in which
  statistical evaluations are superfluous.
  Typically, 
  (\ref{remark-degenerated-limit-distribution-1}) 
  means that
  \begin{eqnarray}
    \label{remark-degenerated-limit-distribution-2}
    L^\prime\big(x,y,\SVMlplz(x)\big)
    \,=\,0\quad \;
    \text{for}\;\;
    P-\textup{a.e. }\,
    (x,y)\in\XY\;
  \end{eqnarray}
  and, therefore, the representer
  theorem \cite[Theorem 5.9]{steinwart2008} implies
  $\SVMlplz(x)=0$ almost surely so that 
  (\ref{remark-degenerated-limit-distribution-2}) implies 
  \begin{eqnarray}
    \label{remark-degenerated-limit-distribution-3}
    L^\prime(x,y,0)
    \,=\,0\quad \;
    \text{for}\;\;
    P-\textup{a.e. }\,
    (x,y)\in\XY\;
  \end{eqnarray}
  For example, (\ref{remark-degenerated-limit-distribution-1})
  implies
  (\ref{remark-degenerated-limit-distribution-2})
  and (\ref{remark-degenerated-limit-distribution-3})  
  if
  $H$ is an RKHS which contains 
  constants and at least one function which
  is not almost surely constant, or if
  $H$ is a universal kernel
  (as in case of the Gaussian Kernel)
  and $X_i$ is not almost surely a constant.
    
  Finally, let us summarize the implications of
  (\ref{remark-degenerated-limit-distribution-2})
  and (\ref{remark-degenerated-limit-distribution-3})
  in case of different loss functions.
  Classification with $Y_i\in\{-1,\,1\}$: In case of
  the logistic loss, the squared loss and
  a slightly smoothed hinge loss, 
  (\ref{remark-degenerated-limit-distribution-3})
  is impossible. Regression: In case of
  the Huber loss and the squared loss,
  (\ref{remark-degenerated-limit-distribution-3})
  implies that $Y_i=0$ almost surely. In case of
  a slightly smoothed $\varepsilon$-insensitve loss,
  (\ref{remark-degenerated-limit-distribution-3})
  implies $Y_i\in[-\varepsilon,\varepsilon]$
  almost surely.
\end{remark}

\subsection{Supplements and Sketch of the Proof}
  \label{sec-sketch-of-proof}

The proof of Theorems \ref{theorem-sqrt-n-consistency}
and \ref{theorem-sqrt-n-consistency-risks}
is an involved
application of the functional delta-method.
In oder to describe this in some more detail, let us first fix
a constant sequence of regularization parameters. That
is, $\lambda_{\mathbf{D}_n}\equiv\lambda_{0}\in(0,\infty)$ for every
$n\in\N$\,. Then, support vector machines may be
represented by a functional $S$ on a set 
of probability measures 
on $\big(\XY,\mathfrak{B}(\XY)\big)$\,.
This functional
$$S\;:\;\;
   P\;\mapsto\;\SVMlplz
$$
is called \emph{SVM-functional} in the following.
It represents the SVM-estimator because
the empirical support vector machine is equal to
$f_{L,D_{n},\lambda_{0}}=S(\mathds{P}_{D_{n}})
$
for every data set $D_{n}\in(\XY)^n$
where $\mathds{P}_{D_{n}}$ denotes the empirical measure
corresponding to $D_{n}$\,. In order to
use the functional delta-method, it is crucial
that this is true for every sample size $n$ and
that $S$ does not depend on $n$\,.
(In Remark \ref{remark-sketch-of-proof-technicalities-b}, 
it will be explained how it is nevertheless
possible to deal with random sequences $\lambda_{\mathbf{D}_{n}}$.)
Theorem \ref{theorem-sqrt-n-consistency} can be shown in the
following way:\\
1. Show that $\sqrt{n}(\mathds{P}_{\mathbf{D}_{n}}-P)$
   converges weakly to a Gaussian process.\\
2. Show that $S$ is Hadamard-differentiable:
   
(a) Show that $S$ is G\^{a}teaux-differentiable.

(b) Show that the G\^{a}teaux-derivative fulfills a
       continuity property.
       
(c) Show that (a) and (b) imply
       Hadamard-differentiability.\\
3. Then, it follows from the functional delta-method that
   $$\sqrt{n}(\SVMldlz-\SVMlplz)\;=\;
     \sqrt{n}\big(S(\mathds{P}_{\mathbf{D}_{n}})-S(P)\big)
   $$
   converges weakly to a Gaussian process.
Theorem \ref{theorem-sqrt-n-consistency-risks} follows from
Theorem \ref{theorem-sqrt-n-consistency} by another
application of the functional delta-method.

\smallskip

Step 1 involves the study of Donsker classes. 
Among other things, this is based on
a bound (\ref{bound-on-uniform-entropy-number}) 
on the uniform entropy number
of balls in the reproducing kernel Hilbert space $H$. 
A proof of this bound is given in
the proof of Lemma \ref{lemma-donsker}. In similar settings,
such bounds have already been proven, e.g.\ in 
\cite[\S\,V]{zhou2003}
and \cite[\S\,6.4]{steinwart2008}. In general, $\sqrt{n}(\mathds{P}_{\mathbf{D}_{n}}-P)$ is \emph{not}
  a measurable random variable so that
  the proof involves the theory
  of weak convergence of unmeasurable random variables; 
  see \cite{vandervaartwellner1996}. However,
  this does not affect the statements of Theorems 
  \ref{theorem-sqrt-n-consistency} and
  \ref{theorem-sqrt-n-consistency-risks}
  because 
$\omega\rightarrow
 f_{L,\mathbf{D}_{n}(\omega),\lambda_{\mathbf{D}_{n}(\omega)}}%
$
  is a measurable random variable 
as shown in the beginning of the proof of 
  Theorem 
  \ref{theorem-sqrt-n-consistency} in Subsection 
  \ref{subsec-donsker-delta}.

Essentially,
it has already been known that
$S$ is G\^{a}teaux-differentiable because
\cite{christmann2004,christmannsteinwart2007} 
derive the influence function of $S$
which is a (special) G\^{a}teaux-derivative. 
Therefore, essential steps of the proof of Step 2(a)
can be adopted from 
\cite{christmann2004,christmannsteinwart2007} and 
\cite[\S\,10.4]{steinwart2008}
but some care is needed
as we also have to deal with signed measures here.
In addition, we also have to deal with 
a sequence of random regularization parameters $\lambda_{\mathbf{D}_n}$
instead of a fixed $\lambda_{0}$; see Remark 
\ref{remark-sketch-of-proof-technicalities-b}.
In Step 2(c) it will be shown that $S$ is
even Hadamard-differentiable (in a specific sense
described in 
Subsection \ref{subsec-hadamard}). This is done because the
application of the delta-method requires 
Hadamard-differentiability. However, 
this might also be useful for other purposes since,
e.g., the chain rule is valid for 
Hadamard-differentiability 
but not for G\^{a}teaux-differentiability.
\cite{christmannvanmessem2008} show Bouligand-differentiability of
the SVM-functional which also allows  
the chain rule.

\begin{remark}[Sequences of random regularization parameters 
     $\lambda_{\mathbf{D}_n}$]
\label{remark-sketch-of-proof-technicalities-b}  

  For a fixed regularization parameter $\lambda_{0}$\,,
  support vector machines can be represented by a functional
  $S:P\mapsto\SVMlplz$ and the delta-method can be applied for
  $S$. However, if we have a sequence of 
  (random) regularization parameters
  $\lambda_{\mathbf{D}_n}$, we get a (random) sequence of functionals
  $$S_{\mathbf{D}_n}\;:\;\;
     P\;\mapsto\;f_{L,P,\lambda_{\mathbf{D}_n}}
  $$
  for which the delta-method cannot be applied offhand.
  This problem can be solved in the following way:
  As described in Subsection \ref{subsec-preparation-for-proof}, 
  $$S_{\mathbf{D}_n}(P)\;=\;f_{L,P,\lambda_{\mathbf{D}_n}}
    \;=\;f_{L,\frac{\lambda_{0}}{\lambda_{\mathbf{D}_n}}P,\lambda_{0}}\;=\;
    S\big({\textstyle \frac{\lambda_{0}}{\lambda_{\mathbf{D}_n}}}P\big)
    \qquad\forall\,P\;.
  $$
  so that everything can be traced back to $S$\,. In this way,
  the explicit use of $S_{\mathbf{D}_n}$ can be avoided and
  the delta-method turns out to be applicable also in this
  case. The price we have to pay is that we have to deal with
  general finite measures in the proofs because, in general,
  $\frac{\lambda_{0}}{\lambda_{\mathbf{D}_n(\omega)}}P$ 
  is not a probability measure any more. 
\end{remark}

\section{Conclusions}\label{sec-conclusions}

In the article, asymptotic properties of support vector machines 
are investigated. For sequences 
of random regularization parameters $\lambda_{\mathbf{D}_n}$, $n\in\N$,
such that
$\sqrt{n}\big(\lambda_{\mathbf{D}_n}-\lambda_{0}\big)\longrightarrow0$
in probability,
it is shown that the difference 
between the empirical and the theoretical SVM 
is asymptotically normal with rate
$\sqrt{n}$; that is, $\sqrt{n}(\SVMldldn-\SVMlplz)$ converges to a
Gaussian process in the function space $H$.
The value $\lambda_{0}>0$  
corresponds to a bound on the complexity of the estimate
for the regression function; a smaller $\lambda_{0}$ 
allows for more complex functions. Therefore,
the theoretical SVM $\SVMlplz$ serves as a ``smoother'' 
approximation of more complex regression functions.
The results of this article show that, in nonparametric 
classification and regression
problems, the estimation of this smoother approximation 
by use of empirical SVMs 
in an infinite dimensional function space
is asymptotically normal
with rate $\sqrt{n}$ -- just as if it was a parametric problem.
The proof is done by showing that the map 
$P\mapsto\SVMlpl$ is suitably Hadamard-differentiable and
by an application of a functional delta-method.

Estimating a smoother approximation of the regression function is
a comprise between a parametric model and a fully non-parametric 
model without any assumptions on the regression function or
the distribution. Without any of such assumptions,
similar results are not possible as follows from the
no-free-lunch theorem.

\section*{Acknowledgment}

I would like to thank Andreas Christmann for 
bringing the problem to my attention and for valuable suggestions.

\section{Appendix: Proof of the Main Results}\label{sec-appendix}

The assumptions of Theorem 
\ref{theorem-sqrt-n-consistency} are valid
in the whole appendix.

\subsection{Preparations}
   \label{subsec-preparation-for-proof}

The map $\Phi:\X\rightarrow H$ always denotes the canonical
feature map corresponding to the kernel $k$ and the RKHS $H$.
It will frequently be used in the proofs that the reproducing property
implies
\begin{eqnarray}\label{feature-map-reproducing-property-pre}
  \la \Phi(x),f\ra\;=\;f(x) \qquad\forall\,x\in\X,\;\;\;\forall\, f\in H
\end{eqnarray}
or, in shorter notation,
\begin{eqnarray}\label{feature-map-reproducing-property}
  \la \Phi,f\ra\;=\;f \qquad\forall\, f\in H\;.
\end{eqnarray}
In particular, we have
\begin{eqnarray}\label{feature-map-reproducing-property-expectation}
  \mathds{E}_{\mu}\la\Phi,f\ra= \!\!
  \int\! \la\Phi,f\ra \,d\mu = \!\!
  \int\! \la\Phi(x),f\ra \,\mu(dx) =\!\!
  \int\!\! f(x) \,\mu(dx).
\end{eqnarray}
According to \cite[p.\ 124]{steinwart2008}, boundedness of
$k$ implies:
\begin{eqnarray}\label{bounded-feature-map}
  \|k\|_{\infty}\;:=\;\sup_{x\in\X}\sqrt{k(x,x)}\;=\;
  \sup_{x\in\X}\big\|\Phi(x)\big\|_{H}\;<\;\infty
\end{eqnarray}
\begin{eqnarray}\label{hilbert-norm-uniform-norm}
  \|f\|_{\infty}\;\;\leq\;\;\|k\|_{\infty}\cdot\|f\|_{H}
  \qquad\forall\,f\in H\;.
\end{eqnarray}
In order to shorten notation, define
$$L_{f}\;:\;\;\XY\;\rightarrow\;\R\,,\qquad
  (x,y)\;\mapsto\;L_{f}(x,y)\,=\,L\big(x,y,f(x)\big)
$$
for every function $f:\X\rightarrow\R$\,. Accordingly,
define 
$$\Ls_{f}(x,y)=\Ls\big(x,y,f(x)\big)\qquad\text{and}\qquad
  \Lss_{f}(x,y)=\Lss\big(x,y,f(x)\big)
$$ 
for every
$(x,y)\in\XY$.
As $L$ is a $P$-square-integrable Nemitski loss function
of order $p\in[1,\infty)$\,,
there is a $b\in L_{2}(P)$ such that
\begin{eqnarray}\label{def-nemitski-loss}
  \big|L(x,y,t)\big|\;\leq\;b(x,y)+|t|^{p}
  \qquad\forall\,(x,y,t)\in\XY\times\R\;.
\end{eqnarray}

Let
$$\mathcal{G}_{1}\;\;:=\;\;
  \big\{g:\XY\rightarrow\R\;
  \big|\;\;\exists\,z\in\R^{d+1}\;\;\text{such that}\;\;g=I_{(-\infty,z]}
  \big\}
$$
be the set of all indicator functions $I_{(-\infty,z]}$. 
Then, it is well-known that
$$\sqrt{n}\big(\mathds{F}_{n}-F\big)
    \;\leadsto\;
    \mathds{G}_{1}
    \qquad\text{in}\quad\ell_{\infty}(\G_{1})
$$
where $\mathds{F}_{n}$ denotes the empirical process,
$F$ denotes the distribution function of $P$,
$\mathds{G}_{1}$ is a Gaussian process, and
$\ell_{\infty}(\G_{1})$ denotes the set of all
bounded functions $G:\G_1\rightarrow\R$. 
Provided that the SVM-functional $S$ is Hadamard-differentiable
in $\ell_{\infty}(\G_{1})$, an application of the functional
delta-method would yield asymptotic normality of 
$\sqrt{n}\big(S(\mathds{F}_{n})-S(F)\big)$.
Unfortunately, the norm-topology of $\ell_{\infty}(\G_{1})$
is too weak in order to ensure Hadamard-differentiability.
Therefore,
the set of indicator functions $\G_{1}$
has to be enlarged to a set $\G\supset\G_{1}$ which leads to
the following somewhat technical definition of the domain
$B_S$ of the SVM-functional $S$.
Define 
\begin{eqnarray}\label{prep-def-c0-constant}
  c_{0}\;:=\;
  \sqrt{\frac{1}{\lambda_{0}}\int b\, \,dP\,}\,+\,1\;,
\end{eqnarray}
$$\mathcal{G}_{2}\;\;:=\;\;
  \left\{g:\XY\rightarrow\R\;
    \Bigg|\;\;
      \begin{array}{c}
        \exists\,f_{0}\in H\,, \;\;\exists\,f\in H 
           \;\;\text{such that} \\
        \|f_{0}\|_{H}\leq c_{0}\,,\,\;
        \|f\|_{H}\leq 1\;\,\text{and}\\
        g=\Ls_{f_{0}}f
      \end{array}
  \right\}\;,
$$
and
$$\G\;\;:=\;\;\G_{1}\cup\G_{2}\cup\{b\}\;.
$$
Let $\ell_{\infty}(\G)$ be the set of all bounded functions
$$F\;:\;\;\G\;\rightarrow\;\R
$$
with norm $\|F\|_{\infty}=\sup_{g\in\G}\big|F(g)\big|$\,.
Define
$$B_{S}\;:=\;
  \left\{F:\G\rightarrow\R\;
    \Bigg|\;\;
      \begin{array}{c}
        \exists\,\mu\not=0\;\text{a finite measure on}\;
                  \XY\;\text{such that} \\
        F(g)=\int g \,d\mu \;\,\forall\,g\in\G \,,\\
        b\in L_{2}(\mu)\,, \;\;
        b_{a}^{\prime}\in L_{2}(\mu)
        \;\,\,\forall\,a\in(0,\infty)
      \end{array}
  \right\}
$$  
and $B_{0}:=\textup{cl}\big(\textup{lin}(B_{S})\big)$ the closed linear
span of $B_{S}$
in $\ell_{\infty}(\G)$\,. That is, $B_{S}$ is a subset of
$\ell_{\infty}(\G)$ whose elements correspond to 
finite measures. The elements of $B_{S}$ can be seen as some kind of
generalized distribution functions.
Note that the assumptions on $L$ and $P$ imply that
$\G\rightarrow\R,\;\;g\mapsto\int g\,dP$
is a well-defined element of $B_{S}$\,.

For every $F\in B_{S}$\,,
let $\iota(F)$ denote the corresponding
finite measure $\mu$ on $\big(\XY,\mathfrak{B}(\XY)\big)$
such that
$$F(g)\;=\;\int g \,d\mu
  \qquad\forall\,g\in\G\;.
$$ 
Note that, by definition of $B_{S}$\,,
$\iota(F)$ uniquely exists for every
$F\in B_{S}$ so that
$$\iota\;:\;\;B_{S}\;\rightarrow\;
  \textup{ca}^{+}(\XY,\mathfrak{B}(\XY))\,,\qquad
  F\;\mapsto\;\iota(F)\;.
$$
is well-defined where $\textup{ca}^{+}(\XY,\mathfrak{B}(\XY))$
denotes the set
of all finite measures 
on $(\XY,\mathfrak{B}(\XY))$.
The set of all finite \emph{signed} measures 
on $(\XY,\mathfrak{B}(\XY))$ is denoted by
$\textup{ca}(\XY,\mathfrak{B}(\XY))$. 
The set
of all continuous functions $f:\X\rightarrow\R$ is denoted by
$\mathcal{C}(\X)$. Since $\X$ is compact by assumption,
the elements of $\mathcal{C}(\X)$ are bounded and
$\mathcal{C}(\X)$ is endowed with the sup-norm
$\|f\|_{\infty}=\sup_{x\in\X}|f(x)|$.

\smallskip

By now, support vector machines are only defined for
probability measures $\tilde{P}$. However,
in order to deal with sequences of
random regularization parameters $\lambda_{\mathbf{D}_n}$, we will
also have to deal with ``support vector machines'' for
general finite measures $\mu$. For every $F\in B_{S}$,
define
$$f_{L,\iota(F),\lambda}\;:=\;
  \text{arg}\inf_{f\in H}
    \int L\big(x,y,f(x)\big)\,\iota(F)\big(d(x,y)\big)
    \,+\,\lambda\|f\|_{H}^{2}\;.
$$
Though 
$\mu:=\iota(F)\in\textup{ca}^{+}(\XY,\mathfrak{B}(\XY))$
is not necessarily a probability measure,
we have, in effect, not defined any new object.
In order to see this, note that dividing the objective
function by $M:=\mu(\XY)$ does not change the minimizer
so that we get
$$f_{L,\mu,\lambda}\;=\;
  \text{arg}\inf_{f\in H}
    \int L\big(x,y,f(x)\big)\,\frac{1}{M}\mu\big(d(x,y)\big)
    \,+\,\frac{\lambda}{M}\|f\|_{H}^{2}
  \;=\;f_{L,\frac{1}{M}\mu,\frac{\lambda}{M}}
$$
and $f_{L,\frac{1}{M}\mu,\frac{\lambda}{M}}$ is an ``ordinary''
support vector machine 
as
$\frac{1}{M}\mu$ is a probability measure.
This also shows that $f_{L,\mu,\lambda}$ uniquely exists
because $f_{L,\frac{1}{M}\mu,\frac{\lambda}{M}}$ 
uniquely exists
for the probability measure $\frac{1}{M}\mu$
according to \cite[Lemma 5.1 and Theorem 5.2]{steinwart2008}. 
 
The idea is that
considering support vector machines for
general finite measures $\mu$ makes it possible
to take $\lambda_{0}$ as a ``standard regularization parameter''.
Define
$$S\;:\;\;B_{S}\;\rightarrow\;H\,,\qquad
  F\;\mapsto\;S(F)\;=\;f_{\iota(F)}
$$
where
$$f_{\iota(F)}\;:=\;f_{L,\iota(F),\lambda_{0}}\;=\;
  \text{arg}\inf_{f\in H}
    \int L\big(x,y,f(x)\big)\,\iota(F)\big(d(x,y)\big)
    \,+\,\lambda_{0}\|f\|_{H}^{2}\;.
$$
Then, we can deal with other regularization parameters 
$\lambda>0$ by use of
\begin{eqnarray}\label{prep-standard-regularization-parameter}
  f_{L,\iota(F),\lambda}\;=\;
  S\big({\textstyle\frac{\lambda_{0}}{\lambda}}F\big)
  \qquad\quad\forall\,F\in B_{S}\;.
\end{eqnarray}
This is important in order to apply the functional 
delta-method in case of a sequence of random regularization
parameters $\lambda_{\mathbf{D}_n}$; see also
Remark \ref{remark-sketch-of-proof-technicalities-b}.

\smallskip

It follows from \cite[Eqn.\ (5.4) and Lemma 4.23]{steinwart2008}
that
\begin{eqnarray}
  \label{prep-inequality-hnorm}
    \big\|f_{\iota(F)}\big\|_{H}\;\leq\;
    \sqrt{\tfrac{1}{\lambda_{0}}F(b)}
    \qquad\forall\,F\in B_{S}\,,\\
  \label{prep-inequality-supnorm}
    \big\|f_{\iota(F)}\big\|_{\infty}\;\leq\;
    \|k\|_{\infty}\sqrt{\tfrac{1}{\lambda_{0}}F(b)}
    \qquad\forall\,F\in B_{S}\;.
\end{eqnarray}
Since $\X$ is separable and $k$ is a continuous kernel,
the RKHS $H$ is a \emph{separable} Hilbert space;
see \cite[Lemma 4.33]{steinwart2008}. 
Separability of $H$ is used several times in the proofs;
this is important particularly
with regard to the Bochner-integral
of $H$-valued functions $\Psi:\mathcal{Z}\rightarrow H$\,. 
The Bochner-integral 
$\int\! \Psi \,d\mu=
  \int\! \Psi \,d\mu^{+}-\int\! \Psi \,d\mu^{-}
$
of such a $H$-valued function $\Psi$ 
with respect to a finite signed measure $\mu=\mu^{+}-\mu^{-}$
is again an element of $H$\,.
If $\Psi$ is suitably measurable, then existence of the
Bochner-integral follows from 
$\int \|\Psi\|_{H} \,d|\mu|<\infty$ where $|\mu|=\mu^{+}+\mu^{-}$
denotes the total variation of $\mu$. We will also 
frequently use the fact that, for every Banach space $E$
and every continuous linear operator $A:H\rightarrow E$,
the existence of the Bochner-integral $\int \Psi \,d\mu$
implies the existence of the Bochner-integral
$\int A(\Psi) \,d\mu$ and
\begin{eqnarray}\label{Bochner-integral-continuous-linear-operator}
  \int A(\Psi) \,d\mu\;=\;A\left(\int \Psi \,d\mu\right)\;;
\end{eqnarray}
see, e.g.\ \cite[Theorem 3.10.16 and Remark 3.10.17]{denkowski2003}.

\smallskip

This subsection closes with three 
lemmas which are used several times.
Thereafter, G\^{a}teaux-differentiability of the
SVM-functional $S:B_{S}\rightarrow H$ will be shown in Subsection
\ref{subsec-gateaux}. This is strengthened to
Hadamard-differ\-entiability in Subsection
\ref{subsec-hadamard}. Finally,
it will be shown in Subsection \ref{subsec-donsker-delta}
that $\sqrt{n}(\mathds{P}_{\mathbf{D}_{n}}-P)$ converges
weakly to a Gaussian process in $\ell_{\infty}(\G)$
and that this implies asymptotic normality
of  
$$\sqrt{n}\big(f_{L,\mathbf{D}_{n},\lambda_{\mathbf{D}_n}}-\SVMlplz\big)
  \quad\;\text{and}\quad\;
  \sqrt{n}\big(\mathcal{R}_{L,P}(f_{L,\mathbf{D}_{n},\lambda_{\mathbf{D}_n}})
                 -\mathcal{R}_{L,P}(\SVMlplz)
          \big)
$$
by applying a functional delta-method.

\begin{lemma}\label{lemma-convergence-in-BS-weak-convergence}
  Let $(F_{n})_{n\in\N}\subset B_{S}$ be a sequence which converges 
  to some $F_{0}\in B_{S}$\,. Then,
  $\lim_{n\rightarrow\infty}\iota(F_{n})(\XY)=
               \iota(F_{0})(\XY)
  $
  and the sequence of finite measures $\iota(F_{n})$,
     $n\in\N$, converges weakly to $\iota(F_{0})$\,.
\end{lemma}
\begin{proof}\item
  Define $M_{n}:=\iota(F_{n})(\XY)$ 
  and $a_{n}=(n,\dots,n)\in\R^{d+1}$
  for every $n\in\N\cup\{0\}$\,.
  Then,
  $$
    0\leq\big|M_{n}-M_{0}\big|
    =\lim_{l\rightarrow\infty}
            \big|F_{n}\big(I_{(-\infty,a_{l}]}\big)
                 -F_{0}\big(I_{(-\infty,a_{l}]}\big)
            \big|
    \leq \big\|F_{n}-F_{0}\big\|_{\infty}
    \longrightarrow 0\,.
  $$
  Therefore, the normalized sequence $\tilde{F}_{n}=M_{n}^{-1}F_{n}$,
  $n\in\N\cup\{0\}$, corresponds to a sequence of probability measures
  $\iota(\tilde{F}_{n})$ such that
  \begin{eqnarray*}
    \lim_{n\rightarrow\infty}\iota(\tilde{F}_{n})\big((-\infty,a]\cap\XY\big)\!
    &=&\!\!\lim_{n\rightarrow\infty}
          \frac{1}{M_{n}}F_{n}\big(I_{(-\infty,a]}\big)
       \;=\;\frac{1}{M_{0}}F_{0}\big(I_{(-\infty,a]}\big)\\
    &=&\!\iota(\tilde{F}_{0})\big((-\infty,a]\cap\XY\big)
  \end{eqnarray*}
  for every $a\in\R^{d+1}$\,. Hence, it follows from the
  Portmanteau theorem that
  the sequence of probability measures 
  $(\iota(\tilde{F}_{n}))_{n\in\N}$ converges weakly
  to $\iota(\tilde{F}_{0})$\,; see e.g.\
  \cite[Lemma 2.2]{vandervaart1998}. 
  Finally, this implies that the sequence of finite measures
  $(\iota(F_{n}))_{n\in\N}$ converges weakly to $\iota(F_{0})$\,.
\end{proof}

\begin{lemma}\label{lemma-prep-well-def-operators}
  For every $G\in\textup{lin}(B_{S})$\,,
  there is a unique finite signed measure
  $\iota(G)=\mu$ on $(\XY,\mathfrak{B}(\XY))$ such that
  \begin{eqnarray}\label{lemma-prep-well-def-operators-1}
    \int g\,d\mu\;=\;G(g)\qquad\forall\,g\in\G\;.
  \end{eqnarray}
  The map
  $$\iota\;:\;\;\textup{lin}(B_{S})\;\rightarrow\;
    \textup{ca}(\XY,\mathfrak{B}(\XY)\,,\qquad
    G\;\mapsto\;\iota(G)\;.
  $$
  defined by (\ref{lemma-prep-well-def-operators-1})
  is linear. Let $G\in\textup{lin}(B_{S})$ and
  $\mu=\iota(G)$\,.
  Then,
  $$b\;\in\;L_{2}\big(|\mu|\big)\,,
    \qquad
    b_{a}^{\prime}\;\in\;L_{2}\big(|\mu|\big)
    \;\;\;\forall\,a\in(0,\infty)
  $$
  and $\Ls_{f}\Phi$ and $\Lss_{f}h\Phi$ are Bochner-integrable
  with respect to $\mu$ for every $f,h\in H$.
  Furthermore,
  \begin{eqnarray*}
    \tilde{A}_{f}\;:\;\;C(\X)\;\rightarrow\;H\,,\qquad
    h\;\mapsto\;\int\Lss_{f}h\Phi\,d\mu\;,\\
    A_{f}\;:\;\;H\;\rightarrow\;H\,,\qquad
    h\;\mapsto\;\int\Lss_{f}h\Phi\,d\mu\;.
  \end{eqnarray*}
  are continuous linear operators for every $f\in H$.
\end{lemma}
\begin{proof}\item
  For every $G\in\textup{lin}(B_{S})$\,, there are 
  $F_{1},F_{2}\in B_{S}$ such that
  $G=F_1-F_2$\,. Define
  $\mu:=\iota(F_1)-\iota(F_2)$.
  Then, $\mu$ fulfills (\ref{lemma-prep-well-def-operators-1}).
  From the definition of $B_{S}$ and
  $$|\mu|(C)\;\leq\;\iota(F_1)(C)+\iota(F_2)(C)
    \qquad\forall\,C\in\mathfrak{B}(\XY)
  $$
  it follows that
  $b,\,b_{a}^{\prime}\,\in\,L_{2}\big(|\mu|\big)
  $
  for every $a\in(0,\infty)$\,.
  Next, fix any $f\in H$ and define 
  $a=\|f\|_{\infty}<\infty$; see
  (\ref{hilbert-norm-uniform-norm}).
  Then,
  $$\int \big\|\Ls_{f}\Phi\|_{H}\,d|\mu|
    \;\stackrel{(\ref{bounded-feature-map})}{\leq}\;
    \|k\|_{\infty}\cdot\int |\Ls_{f}|\,d|\mu|
    \;\stackrel{(\ref{theorem-sqrt-n-consistency-1})}{\leq\;}
    \|k\|_{\infty}\cdot\int b_{a}^{\prime}\,d|\mu|
    \;<\;\infty
  $$
  and, therefore, $\Ls_{f}\Phi$ is Bochner-integrable;
  see e.g.\ \cite[Theorem 3.10.3 and Theorem 3.10.9]{denkowski2003}.
  A similar calculation shows that $\Lss_{f}h\Phi$ is 
  Bochner-integrable, too.
  
  In order to prove uniqueness of $\mu$, let
  $\mu_{1}$ and $\mu_{2}$ be finite signed measures
  such that $\int g\,d\mu_{1}=\int g\,d\mu_{2}$ for every
  $g\in\G$. From this equation it follows that
  $\int g\,d(\mu_{1}^{+}+\mu_{2}^{-})=\int g\,d(\mu_{2}^{+}+\mu_{1}^{-})$
  for every $g\in\G$. Since $\mu_{1}^{+}+\mu_{2}^{-}$ and
  $\mu_{2}^{+}+\mu_{1}^{-}$ are finite (positive) measures
  and $\G$ contains all indicator functions $I_{(-\infty,z]}$,
  $z\in\R^{d+1}$, it follows from the uniqueness theorem
  (e.g.\ \cite[\S\,1.7]{hoffmannjoergensen1994})
  that $\mu_{1}^{+}+\mu_{2}^{-}=\mu_{2}^{+}+\mu_{1}^{-}$.
  Hence, $\mu_{1}=\mu_{2}$.
  
  Uniqueness and (\ref{lemma-prep-well-def-operators-1})
  imply linearity of the map $\iota$.
  
  Now let us turn over to $\tilde{A}_{f}$
  for any fixed $f\in H$\,. Obviously, 
  $\tilde{A}_{f}$ is linear. In order to prove
  that $\tilde{A}_{f}$ is a continuous linear operator,
  define
  $a:=\|f\|_{\infty}$\,,
  which is a finite number due to 
  (\ref{hilbert-norm-uniform-norm}).
  Then,
  $$\big\|\tilde{A}_{f}(h)\|_{H}
    \;\leq\;
      \int\!\big\|\Lss_{f}h\Phi\big\|_{H}\,d|\mu|
    \;\stackrel{(\ref{bounded-feature-map},
       \ref{theorem-sqrt-n-consistency-1})}{\leq}\;
          \|h\|_{\infty}\|k\|_{\infty}\!\!
          \int \!\big|b_{a}^{\prime\prime}\big|\,\,|\mu|\big(d(x,y))
          \;<\;\infty\;.
  $$
  According to \cite[Lemma 4.23]{steinwart2008},
  the canonical embedding $H\rightarrow\mathcal{C}(\X)$ is a 
  continuous linear operator. 
  Hence, it also follows that $A_{f}$ is a continuous linear operator.
\end{proof}

\begin{lemma} \label{lemma-continuous-partial-derivative}
  Let $(\mu_{n})_{n\in\N}$ be a tight sequence of
  finite signed measures on
  $(\XY,\mathfrak{B}(\XY))$ such that
  \,$\sup_{n\in\N}|\mu_{n}|(\XY)<\infty$\,.
  Let $(f_{n})_{n\in\N}\subset H$ be a sequence converging to some
  $f_{0}\in H$\,. Then,
  $$\lim_{n\rightarrow\infty}
    \sup_{h\in H \atop \|h\|_{H}\leq 1}
      \left\|\int L_{f_{n}}^{\prime\prime}h\Phi \,d\mu_{n} -
             \int L_{f_{0}}^{\prime\prime}h\Phi \,d\mu_{n}
      \right\|_{H}\;=\;0\;.
  $$
\end{lemma}
\begin{proof}\item 
  For every $\varepsilon>0$ there is a compact subset 
  $\mathcal{Z}_{\varepsilon}\subset\XY$ such that
  \begin{eqnarray}\label{lemma-continuous-partial-derivative-p1}
    |\mu_{n}|(\XY\setminus \mathcal{Z}_{\varepsilon})\;<\;\varepsilon
    \qquad\forall\,n\in\N\;.
  \end{eqnarray}
  Define
  $a:=\sup_{n\in\N_{0}}\|f_{n}\|_{\infty}
   \!\stackrel{(\ref{hilbert-norm-uniform-norm})}{\leq}\! 
   \|k\|_{\infty}\sup_{n\in\N_{0}}\|f_{n}\|_{H}<\infty
  $\,.
  For every $n\in\N$\,,
  \begin{eqnarray*}
    \lefteqn{
      \sup_{h\in H \atop \|h\|_{H}\leq 1}
      \left\|\int\!\! L_{f_{n}}^{\prime\prime}h\Phi \,d\mu_{n} -
             \int\!\! L_{f_{0}}^{\prime\prime}h\Phi \,d\mu_{n}
      \right\|_{H}=
       \sup_{h\in H \atop \|h\|_{H}\leq 1}
       \left\|\int \!\! 
                   \big(L_{f_{n}}^{\prime\prime}\!-L_{f_{0}}^{\prime\prime}
                   \big)
                   h\Phi 
              \,d\mu_{n} 
       \right\|_{H} }\\
    &\leq&
       \sup_{h\in H \atop \|h\|_{H}\leq 1}
       \int \big|L_{f_{n}}^{\prime\prime}(x,y)
                        -L_{f_{0}}^{\prime\prime}(x,y)
            \big|\cdot
                   \|h\|_{\infty}\cdot\big\|\Phi(x)\big\|_{H} 
       \,\,|\mu_{n}|\big(d(x,y)\big)
       \leq \\
    &\stackrel{(\ref{bounded-feature-map},
                 \ref{hilbert-norm-uniform-norm})}{\leq}&
       \|k\|_{\infty}^{2}
       \int \big|L_{f_{n}}^{\prime\prime}(x,y)
                        -L_{f_{0}}^{\prime\prime}(x,y)
            \big| 
       \,\,|\mu_{n}|\big(d(x,y)\big)  \leq\\
    &\stackrel{(\ref{theorem-sqrt-n-consistency-1},
                \ref{lemma-continuous-partial-derivative-p1}
                )}{\leq}&
       \|k\|_{\infty}^{2}
       \int_{\mathcal{Z}_{\varepsilon}} 
            \big|L_{f_{n}}^{\prime\prime}(x,y)
                        -L_{f_{0}}^{\prime\prime}(x,y)
            \big| 
       \,\,|\mu_{n}|\big(d(x,y)\big)
       \,\,+\,\,2\|k\|_{\infty}^{2}b_{a}^{\prime\prime}\varepsilon \;\leq\\
    &\leq&\|k\|_{\infty}^{2}|\mu_{n}|(\XY)
        \sup_{(x,y)\in\mathcal{Z}_{\varepsilon}}
            \big|L_{f_{n}}^{\prime\prime}(x,y)
                        -L_{f_{0}}^{\prime\prime}(x,y)
            \big|
       \,\,+\,\,2\|k\|_{\infty}^{2}b_{a}^{\prime\prime}\varepsilon  
  \end{eqnarray*}
  Since \,$\sup_{n\in\N}|\mu_{n}|(\XY)<\infty$\, and
  $\varepsilon>0$ can be chosen arbitrarily small,
  it only remains to prove that
  \begin{eqnarray}\label{lemma-continuous-partial-derivative-p2}
      \lim_{n\rightarrow\infty}  
      \sup_{(x,y)\in\mathcal{Z}_{\varepsilon}}
            \big|L_{f_{n}}^{\prime\prime}(x,y)
                        -L_{f_{0}}^{\prime\prime}(x,y)
            \big|
      \;\;=\;\;0
  \end{eqnarray}
  Continuity of $\Lss$ and compactness of
  $\mathcal{Z}_{\varepsilon}\times[-a,a]$
  imply that $\Lss$ is uniformly continuous on
  $\mathcal{Z}_{\varepsilon}\times[-a,a]$\,.
  Assertion (\ref{lemma-continuous-partial-derivative-p2})
  is an easy consequence of uniform continuity
  of $\Lss$ on
  $\mathcal{Z}_{\varepsilon}\times[-a,a]$\,, inequality
  $-a\leq f_{n}\leq a$ for every $n\in\N_{0}$, and
  the fact that
  $\lim_{n}\|f_{n}-f_{0}\|_{H}=0$ implies
  $\lim_{n}\|f_{n}-f_{0}\|_{\infty}=0$\,.
\end{proof}

\subsection{G\^{a}teaux-Differentiability of the SVM-Functional}
    \label{subsec-gateaux}

In this subsection, it will be shown that the SVM-functional
$$S\;:\;\;B_{S}\;\rightarrow\;H\,,\qquad
  F\;\mapsto\;f_{\iota(F)}
$$
is G\^{a}teaux-differentiable. 
Essentially,
this has already been known because
\cite{christmann2004,christmannsteinwart2007} 
derive the influence function of $S$
which is a (special) G\^{a}teaux-derivative. 
Therefore, the proofs in this subsection 
can essentially be adopted from 
\cite{christmann2004,christmannsteinwart2007} and 
\cite[\S\,10.4]{steinwart2008}.
However, some care is needed
as we also have to deal with signed measures and 
with 
a (random) sequence of regularization parameters 
$\lambda_{\mathbf{D}_n}$
instead of a fixed $\lambda_{0}$; see also Remark 
\ref{remark-sketch-of-proof-technicalities-b}.

At first, we have to show Fr\'{e}chet-differentiability of
the ``generalized risk'' 
$\mathcal{R}_{L,\mu}:f\mapsto\int L_{f}\,d\mu$
(and of its derivative) for finite signed measures $\mu$\,.
If $\mu$ is a probability measure, then
Lemma \ref{lemma-derivatives}(a) is just
the well-known Fr\'{e}chet-differentiability of the
ordinary risk $\mathcal{R}_{L,P}$\,. 

\begin{lemma}\label{lemma-derivatives}
  For every finite signed measure $\mu$ on
  $(\XY,\mathfrak{B}(\XY))$
  such that
  \begin{eqnarray}\label{lemma-derivatives-1}
    \int b \,d|\mu|\,\,<\,\,\infty
    \qquad\text{and}\qquad
    \int b_{a}^{\prime} \,d|\mu|\,\,<\,\,\infty\,\quad\forall\,
    a\in(0,\infty)\,,\;\;
  \end{eqnarray}
  the following statements are true:
  \begin{enumerate}
   \item[\textbf{(a)}] The map
     $$H\;\rightarrow\;\R\,,\qquad
       f\;\mapsto\;\int L_{f} \,d\mu
     $$
     is Fr\'{e}chet-differentiable and its
     Fr\'{e}chet-derivative in $f\in H$ is given by
     $H\rightarrow\R,\;\;
       h\mapsto\big\la\int L_{f}^{\prime}\Phi \,d\mu\,,\,h \big\ra
     $.
   \item[\textbf{(b)}] The map
     $$H\;\rightarrow\;H\,,\qquad
       f\;\mapsto\;\int L_{f}^{\prime}\Phi \,d\mu
     $$
     is Fr\'{e}chet-differentiable and its
     Fr\'{e}chet-derivative in $f\in H$ is given by 
     $H\rightarrow H,\;\;
       h\mapsto
       \int L_{f}^{\prime\prime}h\Phi \,d\mu
     $.
 \end{enumerate}
\end{lemma}
\begin{proof}\item
  Both statements can be proven essentially by following the
  lines of \cite[Lemma 2.21]{steinwart2008}. 
  Since the proofs of (a) and (b) nearly coincide, only
  the proof of (b) is given in detail. 
  
  Define
  $$T(f)\,=\,\int L_{f}^{\prime}\Phi \,d\mu
    \qquad\text{and}\qquad
    T_{f}^{\prime}(h)\,=\,\int L_{f}^{\prime\prime}h\Phi \,d\mu
  $$
  for every $f,h\,\in\,H$\;. Lemma \ref{lemma-prep-well-def-operators}
  guarantees that these Bochner-integrals exist and that 
  $T_{f}^{\prime}:H\rightarrow H,\;h\mapsto T_{f}^{\prime}(h)$ 
  is a continuous linear operator.
  Now, fix any $f\in H$ and let $(h_{n})_{n\in\N}\subset H\setminus\{0\}$ be 
  a sequence
  which converges to $0$ in $H$\,. Define
  $$\gamma_{n}(x,y)\,:=\,
    \frac{\big|\Ls\big(x,y,f(x)\!+\!h_{n}(x)\big)
               -\Ls\big(x,y,f(x)\big)
               -h_{n}(x)\Lss\big(x,y,f(x)\big)
          \big|%
         }{|h_n(x)|}
  $$
  for every $(x,y)\in \XY$ such that $h_n(x)\not=0$ and
  $\gamma_{n}(x,y)=0$ for every $(x,y)\in \XY$ such that $h_n(x)=0$\,.
  The maps $\gamma_{n}:\XY\rightarrow\R,\;\,(x,y)\mapsto\gamma_{n}(x,y)$\,,
  $n\in\N$\,, are measurable. Since $H$ is a RKHS, 
  $\lim_{n\rightarrow\infty}h_{n}(x)=0$ for every $x\in\X$\,. Therefore,
  the definition of $\Ls$ as a partial derivative of $L$ implies
  \begin{eqnarray}\label{lemma-derivatives-p1}
    \lim_{n\rightarrow\infty}\gamma_{n}(x,y)\;=\;0
    \qquad\forall\,(x,y)\in\XY\;.
  \end{eqnarray} 
  Define 
  $a:=\|f\|_{\infty}+\sup_{n\in\N}\|h_n\|_{\infty}
    \stackrel{(\ref{hilbert-norm-uniform-norm})}{\leq} 
    \|k\|_{\infty}
    \big(\|f\|_{H}+\sup_{n\in\N}\|h_n\|_{H}\big)
    <\infty
  $\,.   
  Then, by use of the elementary mean value
  theorem,
  $$\big|\gamma_{n}(x,y)\big|
    \leq \!
      \frac{\big|\Ls\big(x,y,f(x)\!+\!h_{n}(x)\big)
               -\Ls\big(x,y,f(x)\big)
            \big|%
           }{|h_n(x)|}
      +\big|\Lss\big(x,y,f(x)\big)\big| \!
    \stackrel{(\ref{theorem-sqrt-n-consistency-1})}{\leq}
     2 b_{a}^{\prime\prime}
  $$
  for every $(x,y)$ such that $h_n(x)\not=0$ and every $n\in\N$\,.
  Hence, we can use the dominated convergence
  theorem 
  (e.g. \cite[Theorem 4.3.5]{dudley2002})
  in order to finish the proof:
  \begin{eqnarray*}
    \lefteqn{
      \lim_{n\rightarrow\infty}
      \frac{\big\|T(f+h_{n})-T(f)-T_{f}^{\prime}(h_{n})\big\|_{H}}%
           {\|h_{n}\|_{H}}\;\leq
    }\\ 
    &\leq&\lim_{n\rightarrow\infty}
            \int \frac{|h_{n}(x)|}{\|h_{n}\|_{H}}\cdot
                 \big|\gamma_{n}(x,y)\big|\cdot 
                 \big\|\Phi(x)\big\|_{H}
            \,\,|\mu|\big(d(x,y)\big) \;\leq\\ 
    &\stackrel{(\ref{bounded-feature-map},
                 \ref{hilbert-norm-uniform-norm})}{\leq}&
           \lim_{n\rightarrow\infty}
            \|k\|_{\infty}^{2}
            \int \big|\gamma_{n}(x,y)\big|
            \,\,|\mu|\big(d(x,y)\big) 
        \;\stackrel{(\ref{lemma-derivatives-p1})}{=}\;0 
  \end{eqnarray*}
\end{proof}

\begin{lemma}\label{lemma-KF-invertible}
  For every $F\in B_{S}$\,,
  $$K_{F}\;:\;\;H\;\rightarrow\;H\,,\qquad
    f\;\mapsto\;2\lambda_0 f+\int\Lss_{\SVMif}f\Phi\,d[\iota(F)]
  $$
  is a continuous linear operator which is invertible.
\end{lemma}
\begin{proof}\item
  It follows from Lemma \ref{lemma-prep-well-def-operators}
  that $K_{F}$ is a continuous linear operator and it only remains to 
  prove that $K_{F}$ is invertible. 
  This is done by use of the
  Fredholm alternative (see e.g. \cite[Theorem 9.29]{griffel2002}). 
  The following 
  proof is essentially a variant
  of the proof of \cite[Theorem 10.18]{steinwart2008}. 
  We have to show:
  \begin{enumerate}
   \item[(i)] $K_{F}$ is injective.
   \item[(ii)] $A:=A_{f_{\iota(F)}}$ 
     as defined in
     Lemma \ref{lemma-prep-well-def-operators} is a compact operator. 
  \end{enumerate}
  Define $\mu=\iota(F)$\,.
  In order to prove (i), fix any $f\in H\setminus\{0\}$
  and note that convexity of $L$ implies
  $\Lss_{f_{\mu}}\geq0$\,.
  Therefore, 
  \begin{eqnarray*}
    \lefteqn{
      \|K_{F}(f)\|_{H}^2
      \;=\;\big\la 2\lambda_0 f+A(f)\,,\,2\lambda_0 f+A(f)\big\ra \;=
    }\\
    &=&\!\!4\lambda_0^2\|f\|_{H}^2+4\lambda_0\la f,A(f)\ra+\|A(f)\|_{H}^2\;>\;
       4\lambda_0\la f,A(f)\ra\;= \\
    &=&\!\!4\lambda_0\big\la f\,,\,{\textstyle \int}
                                     \Lss_{f_{\mu}}f\Phi \,d\mu
                    \big\ra  
       \stackrel{(\ref{Bochner-integral-continuous-linear-operator})}{=}
       4\lambda_0 \!\!\int \!\! \Lss_{f_{\mu}}f\big\la f,\Phi 
                      \big\ra \,d\mu
       =4\lambda_0 \!\!\int\!\! \Lss_{f_{\mu}}\!\cdot\!f^2 \,d\mu
       \,\geq\,0.
  \end{eqnarray*} 
  In the following, (ii) will be shown.  
  To this end, let $M\subset H$ be a (norm-)bounded subset of $H$\,.
  Since $\X$ is compact, it follows from 
  \cite[Corollary 4.31]{steinwart2008} that $M$ is 
  a relatively compact subset of $\mathcal{C}(\X)$
  (with respect to the norm-topology of $\mathcal{C}(\X)$). 
  In order to prove compactness of $A$\,, we have to show
  that every sequence
  $(A(f_{j}))_{j\in\N}\subset\{A(f)|\,f\in M\}$ contains a 
  convergent subsequence. Relative compactness of $M$
  (in $\mathcal{C}(\X)$) 
  implies that there is
  a subsequence 
  $(f_{j_{\ell}})_{\ell\in\N}\subset(f_{j})_{j\in\N}$
  which is a Cauchy-sequence in $\mathcal{C}(\X)$\,.
  Since $\tilde{A}_{f_{\iota(F)}}$ is a continuous linear
  operator on $\mathcal{C}(\X)$ (Lemma \ref{lemma-prep-well-def-operators}),
  this implies that the sequence
  $$A(f_{j_{\ell}})\;=\;\tilde{A}_{f_{\iota(F)}}(f_{j_{\ell}})\,,
    \qquad\ell\in\N\,,
  $$
  is a Cauchy-sequence in $H$\,.
  Hence, $(A(f_{j_{\ell}}))_{\ell\in\N}$ converges in $H$ since
  $H$ is complete.
\end{proof}

By use of these preliminary lemmas, G\^{a}teaux-differentiability
of the SVM-functional can be shown now:
\begin{proposition}\label{prop-gateaux}
  Let $F\in B_{S}$\,, $G\in\ell_{\infty}(\G)$ and
  $\rho>0$ such that
  $F+sG\in B_{S}$ for every $s\in(-\rho,\rho)$.
  Then,
  there is a unique finite signed measure $\mu$ such that
  \begin{eqnarray}\label{lemma-gateaux-1}
      \int g\,d\mu\;=\;G(g)\qquad\forall\,g\in\G\;.
  \end{eqnarray}
  Furthermore,
  $$\lim_{s\rightarrow 0}
    \left\|\frac{S(F+sG)-S(F)}{s}-S_{F}^{\prime}(G)\right\|_{H}
    \;\;=\;\;0
  $$
  where
  \begin{eqnarray}\label{lemma-gateaux-2}
    S_{F}^{\prime}(G)\;=\;
    -K_{F}^{-1}\Big(\E_{\mu}\big(\Ls_{\SVMif}\Phi\big)\Big)\;.
  \end{eqnarray}
  In particular, $S$ is G\^{a}teaux-differentiable.  
\end{proposition}
\begin{proof}\item
    The following proof is similar to the proof of 
    \cite[Theorem 10.18]{steinwart2008} but some care is needed because
    we also have to deal with \emph{signed} measures here.
    
  \textit{Part 1:}
    Define $\nu:=\iota(F)$\;. Since
    $G=s^{-1}\big((F+sG)-F\big)
      \in\textup{lin}(B_{S})
    $
    for any $s\in(-\rho,\rho)\setminus\{0\}$\,,
    it follows from Lemma \ref{lemma-prep-well-def-operators} that
    there is a unique finite signed measure $\mu$ such that
    \begin{eqnarray}\label{lemma-gateaux-p1}
      \int g\,d\mu\;=\;G(g)\qquad\forall\,g\in\G\;.
    \end{eqnarray}
    Define
    $$\Gamma\;:\;\;\R\times H\;\rightarrow\;H\,,\qquad
      (s,f)\;\mapsto\;2\lambda_0 f + \int \Ls_{f}\Phi\,d\nu
      + s\int \Ls_{f}\Phi\,d\mu\;.
    $$
    Lemma \ref{lemma-derivatives}\,(b) implies that
    the maps
    $\,H\rightarrow H\,,\;\,f\mapsto\int \Ls_{f}\Phi\,d\nu\,$
    and $\,H\rightarrow H\,,\;\,f\mapsto\int \Ls_{f}\Phi\,d\mu\,$
    are continuous. Hence, an
    easy calculation shows that $\Gamma$ is continuous.

  \textit{Part 2:}
    In this part, it will be shown that 
    $\Gamma$ is continuously Fr\'{e}chet-differ\-entiable.  
    First, it follows from Lemma \ref{lemma-derivatives}\,(b)
    that the map
    $$\R\times H\;\rightarrow\;H\,,\qquad
      (s,f)\;\mapsto\;
      \frac{\partial \Gamma}{\partial s}(s,f)\;=\;\int \Ls_{f}\Phi\,d\mu
    $$
    is continuous.   
    Secondly, Lemma \ref{lemma-derivatives}\,(b) yields
    that the partial derivative
    $\frac{\partial \Gamma}{\partial H}(s,f)$
    is given by
    $$\frac{\partial \Gamma}{\partial H}(s,f)\;:\;\;
      H\;\rightarrow\;H\,,\qquad
      h\;\mapsto\;2\lambda_0 h
           +\int \!\!L_{f}^{\prime\prime}h\Phi \,d\nu
           +s\!\int \!\!L_{f}^{\prime\prime}h\Phi \,d\mu      
    $$
    for every $(s,f)\in\R\times H$\,.
    Let $\mathcal{B}(H,H)$ be the set of all continuous linear operators
    $T:H\rightarrow H$\,;  
    this is a Banach space with the operator norm. It follows
    from Lemma \ref{lemma-continuous-partial-derivative} that 
    $$\R\times H\;\rightarrow\;\mathcal{B}(H,H)\,,\qquad
      (s,f)\;\mapsto\;\frac{\partial \Gamma}{\partial H}(s,f)
    $$
    is continuous.  
    Since $\Gamma$ is continuous (as stated above),
    this implies that $\Gamma$ is
    continuously Fr\'{e}chet-differentiable
    according to \cite[p.\ 635]{denkowski2003}.

  \textit{Part 3:}
    Now, we can prove the statement of the lemma
    by use of an implicit function theorem.
    It follows from Lemma \ref{lemma-derivatives}\,(a) 
    that
    \begin{eqnarray}\label{lemma-gateaux-p2}
      \Gamma(s,f)\;=\;
      \frac{\partial\Rs_{L,\nu+s\mu,\lambda_0}}{\partial H}(f)
      \qquad\forall\,f\in H\qquad\forall\,s\in(-\rho,\rho)\;.
    \end{eqnarray}
    Since $H\rightarrow\R\,,\;f\mapsto\Rs_{L,\nu+s\mu,\lambda_0}$
    is strictly convex and continuously Fr\'{e}chet-differentiable, the
    following assertion is valid for every
    $s\in(-\rho,\rho)$:
    \begin{eqnarray}\label{lemma-gateaux-p3}
      \Gamma(s,f)\;=\;0\qquad\Leftrightarrow\qquad
      f=f_{\nu+s\mu}\;.
    \end{eqnarray}
    (Direction ``$\Leftarrow$'' follows from
    \cite[Theorem 7.4.1]{luenberger1969} and
    ``$\Rightarrow$'' follows from
    \cite[Lemma 8.7.1]{luenberger1969}
    and uniqueness of the minimizer.)
    As shown in Part 2, $\Gamma$ is continuously 
    Fr\'{e}chet-differentiable. According to
    Lemma \ref{lemma-KF-invertible},
    $$\frac{\partial \Gamma}{\partial H}\big(0,f_{\nu}\big)
      \;=\;K_{F}
    $$
    is an invertible operator.  
    Therefore, it follows from a classical implicit function theorem 
    (e.g.\ \cite[\S\,4]{akerkar1999})
    that there is a $\delta\in(0,\rho)$ and a Fr\'{e}chet-differentiable map
    $\,\varphi:\,(-\delta,\delta)\rightarrow H\,$ such that
    \begin{eqnarray}\label{lemma-gateaux-p4}
      \Gamma\big((s,\varphi(s)\big)\;=\;0
      \qquad\forall\,s\in (-\delta,\delta)
    \end{eqnarray}
    and the derivative is equal to
    $$\varphi^{\prime}(0)\;=\;
      -\left(\frac{\partial \Gamma}{\partial H}\big(0,\varphi(0)\big)
       \right)^{-1}
       \left(\frac{\partial \Gamma}{\partial s}\big(0,\varphi(0)\big)
       \right)
       \;=\;-K_{F}^{-1}\Big(\E_{\mu}\big(\Ls_{f_{\nu}}\Phi\big)\Big)\;.
    $$
    According to
    (\ref{lemma-gateaux-p3}) and (\ref{lemma-gateaux-p4})\,,
    $\varphi(s)=
      f_{\nu+s\mu}=
      S(F+sG)
    $
    for every $s\in(-\delta,\delta)$.
    Define $S_{F}^{\prime}(G)=\varphi^{\prime}(0)$.
    Hence,
    $$\lim_{s\rightarrow 0}
      \left\|\frac{S(F+sG)-S(F)}{s}-S_{F}^{\prime}(G)\right\|_{H}
      =\;
         \lim_{s\rightarrow 0}
         \left\|\frac{\varphi(s)-\varphi(0)}{s}-\varphi^{\prime}(0)
         \right\|_{H}
      =\; 0\,.
    $$ 
\end{proof}

\subsection{Hadamard-Differentiability of the SVM-Functional}
  \label{subsec-hadamard}

In this subsection, the result of the previous Subsection
\ref{subsec-gateaux} is strengthened.
In statistics, three different types of differentiability
in Banach spaces are particularly important:
G\^{a}teaux-differentiability, 
Hadamard-differentiability and
Fr\'{e}chet-differentiability. Among these,
G\^{a}teaux is the weakest and
Fr\'{e}chet is the strongest notion of 
differentiability. 
In order to apply the functional delta-method, we need the
intermediate Hadamard-differentiability.
It is well-known that 
a G\^{a}teaux-differentiable function 
is even Fr\'{e}chet-differ\-entiable 
(and, therefore, Hadamard-differentiable) if the
(G\^{a}teaux-)derivative is continuous.
In the following Lemma \ref{lemma-prep-hadamard},
it will be shown that the G\^{a}teaux-derivative of $S$
fulfills a certain continuity property
(\ref{lemma-prep-hadamard-1}). This
property is not strong enough in order to guarantee
Fr\'{e}chet-differentiability. However,
it will be shown in the proof of
Theorem \ref{theorem-hadamard} that it is just strong enough
in order to guarantee Hadamard-differentiability
of $S$ tangentially to the closed linear span of $B_{S}$. 
In order to do this, we only have to slightly change
the proof of the well-known interrelationship
between G\^{a}teaux- and Fr\'{e}chet-differentiability
(as provided, e.g., by \cite[Prop.\ 5.1.8]{denkowski2003}).

\begin{lemma}\label{lemma-prep-hadamard}
  Let $B_{0}=\textup{cl}\big(\textup{lin}(B_{S})\big)$
  be the closed linear span of $B_S$ in $\ell_{\infty}(\G)$.
  Let $(G_{n})_{n\in\N}\subset\textup{lin}(B_{S})$
  be a sequence
  such that
  $\lim_{n\rightarrow\infty}\|G_{n}-G_{0}\|_{\infty}=0$
  for some $G_{0}\in\ell_{\infty}(\G)$
  and let $(F_{n})_{n\in\N}\subset B_{S}$
  be a sequence
  such that
  $\lim_{n\rightarrow\infty}\|F_{n}-F_{0}\|_{\infty}=0
  $
  for some $F_{0}\in B_{S}$ which fulfills
  \begin{eqnarray}\label{lemma-prep-hadamard-0}
    F_{0}(b)\;<\;\int b \,dP\,+\,\lambda_{0}\;.
  \end{eqnarray} 
  Then, there is a $n_{0}\in\N$ such that, for every 
  $F\in\{F_{n}|n\in\N_{\geq n_{0}}\}\cup\{F_{0}\}$\,, the map
  $S_{F}^{\prime}:G\mapsto S_{F}^{\prime}(G)$ defined in
  Proposition \ref{prop-gateaux} 
  can be extended to a continuous linear operator
  $S_{F}^{\prime}:\;B_{0}\rightarrow H.
  $
  In addition,
  \begin{eqnarray}\label{lemma-prep-hadamard-1}
    \lim_{n\rightarrow\infty}
    \big\|S_{F_{n}}^{\prime}(G_{n})-S_{F_{0}}^{\prime}(G_{0})\big\|_{H}
    \;=\;0\;.
  \end{eqnarray} 
\end{lemma}
\begin{proof}\item The proof consists of four parts:

 \textit{Part 1:}
  Fix any $F\in B_{S}$ such that
  $\big\|\SVMif\big\|_{H}\leq c_{0}$
  where $c_{0}$ is defined as in (\ref{prep-def-c0-constant}). 
  That is,
  \begin{eqnarray}\label{lemma-prep-hadamard-p1}
    \;\;\Ls_{\SVMif}f\;\;\in\;\;\G
    \qquad\forall\,f\in H\;\;\text{with}\;\;\|f\|_{H}\leq 1\;.     
  \end{eqnarray}
  According to Lemma \ref{lemma-prep-well-def-operators},  
  the map $S_{F}^{\prime}:G\mapsto S_{F}^{\prime}(G)$ defined in
  Proposition \ref{prop-gateaux} can be extended to the map
  $$S_{F}^{\prime}\;:\;\;\textup{lin}(B_{S})\;\rightarrow\;H\,,\qquad
    G\;\rightarrow\;
    -K_{F}^{-1}\Big(\E_{\iota(G)}\big(\Ls_{\SVMif}\Phi\big)\Big)
  $$
  Since $\iota$ is linear according to Lemma 
  \ref{lemma-prep-well-def-operators}, this map is linear.
  In order to prove that $S_{F}^{\prime}$
  is a continuous linear operator on $\textup{lin}(B_{S})$\,,
  it is enough to show that
  $$W_{F}\;:\;\;\textup{lin}(B_{S})\;\rightarrow\;H\,,\qquad
    G\;\rightarrow\;
    \E_{\iota(G)}\big(\Ls_{\SVMif}\Phi\big)
  $$
  is a continuous linear operator because
  $K_{F}^{-1}$ is a continuous linear operator
  according to Lemma \ref{lemma-KF-invertible}.
  To this end, note that for every 
  $G\in\textup{lin}(B_{S})$ and every $f\in H$ such that
  $\|f\|_{H}\leq 1$\,,
  \begin{eqnarray*} 
    \Big\la \E_{\iota(G)}\big(\Ls_{\SVMif}\Phi\big) \,,\,f\Big\ra
    \stackrel{(\ref{Bochner-integral-continuous-linear-operator},
              \ref{feature-map-reproducing-property-expectation})}{=}
    \E_{\iota(G)}\big(\Ls_{\SVMif}f\big)
    \stackrel{(\ref{lemma-prep-well-def-operators-1},
                 \ref{lemma-prep-hadamard-p1})}{=}\;\;
    G\big(\Ls_{\SVMif}f\big) \;.
  \end{eqnarray*}       
  That is, for every $f\in H$ such that
  $\|f\|_{H}\leq 1$\,,
  \begin{eqnarray}\label{lemma-prep-hadamard-p2}
    \big\la W_{F}(G),f\big\ra\;=\;G\big(\Ls_{\SVMif}f\big)
    \quad\;\;\forall\,G\in\textup{lin}(B_{S})\;.\;\;
  \end{eqnarray}
  Hence,
  $$\big\|W_{F}(G)\big\|_{H}\;=\;
    \sup_{f\in H \atop \|f\|_{H}\leq 1} \big\la W_{F}(G),f\big\ra
    \;\stackrel{(\ref{lemma-prep-hadamard-p2})}{=}\;
    \sup_{f\in H \atop \|f\|_{H}\leq 1}G\big(\Ls_{\SVMif}f\big)
    \;\leq\;\|G\|_{\infty}
  $$
  and, therefore, $W_{F}$ is a continuous linear operator
  with operator norm 
  $$\big\|W_{F}\big\|\leq 1\;. 
  $$
  Since $\textup{lin}(B_{S})$ is dense in $B_{0}$\,,
  $W_{F}$ can be extended to a continuous linear operator
  $W_{F}:B_{0}\rightarrow H$
  with operator norm
  \begin{eqnarray}\label{lemma-prep-hadamard-p3}
    \big\|W_{F}\big\|\;\leq\;1\; , \;\;
  \end{eqnarray} 
  see e.g.\ \cite[Theorem 1.9.1]{megginson1998}.
  Hence, $S_{F}^{\prime}$ can be extended to the
  continuous linear map
  $$S_{F}^{\prime}\;:\;\;B_{0}\;\rightarrow\;H\,,
    \qquad G\;\mapsto\;-K_{F}^{-1}\big(W_{F}(G)\big)
  $$
  on $B_{0}=\textup{cl}\big(\textup{lin}(B_{S})\big)$\,.  
  In particular, the latter is eventually true for
  $F=F_{n}$ because it follows from
  $\lim_{n\rightarrow\infty}\|F_{n}-F_{0}\|_{\infty}=0$\,,
  $b\in\G$\,, (\ref{prep-def-c0-constant}), (\ref{prep-inequality-hnorm})
  and (\ref{lemma-prep-hadamard-0})
  that there is some $n_{0}\in\N$ such that
  $$\big\|\SVMifn\big\|_{H}\;\leq\;c_{0}
    \qquad\forall\,n\in\N_{\geq n_{0}}\cup\{0\}\;.
  $$
  and, therefore, $F=F_{n}$ fulfills (\ref{lemma-prep-hadamard-p1})
  for every $n\in\N_{\geq n_{0}}\cup\{0\}$\,.
  
  In addition, note that, for every $G\in B_0$, there is
  a sequence $G_n\in\textup{lin}(B_{S})$, $n\in\N$,
  which converges to $G$ and, therefore,
  $$\big\la W_{F_0}(G),f\big\ra\,=
    \lim_{n\rightarrow\infty}\big\la W_{F_0}(G_n),f\big\ra
    \stackrel{(\ref{lemma-prep-hadamard-p2})}{=}
      \lim_{n\rightarrow\infty}G_n\big(\Ls_{\SVMifz}f\big)
    \,=\,G\big(\Ls_{\SVMifz}f\big)
  $$ 
  for every $f\in H$ such that $\|f\|_H\leq 1$.
  As $K_{F_0}$ is invertable,
  $S_{F_0}^{\prime}(G)=0$ if and only if
  $W_{F_0}(G)\big)=0$.
  Summing up, we may record for
  later purposes (Proposition 
  \ref{prop-degenerated-limit})
  that, 
  for every $G\in B_0$,
  \begin{eqnarray}\label{lemma-prep-hadamard-p3001}
    S_{F_0}^{\prime}(G)\;=\;0
    \quad\;\Leftrightarrow\;\quad
    G\big(\Ls_{\SVMifz}f\big)=0
    \;\;\;\forall\,
    f\in H\text{ such that }
    \|f\|_{H}\leq 1.\;
  \end{eqnarray}
  
 \textit{Part 2:}
  In this part of the proof, it will be shown that
  \begin{eqnarray}\label{lemma-prep-hadamard-p4}
    K_{F_{n}}^{-1}\;\;\xrightarrow[\;n\rightarrow\infty\;]{}\;\;
    K_{F_{0}}^{-1}
    \qquad\text{in the operator norm}\;.
  \end{eqnarray}
  To this end, it suffices to show that
  $$K_{F_{n}}\;\;\xrightarrow[\;n\rightarrow\infty\;]{}\;\;
    K_{F_{0}}
    \qquad\text{in the operator norm}
  $$
  according to \cite[Lemma VII.6.1]{dunford1958}.
  Because of
  \begin{eqnarray*}
    \big\|K_{F_{n}}(f)-K_{F_{0}}(f)\big\|_H
    &\leq& \bigg\|\int\!\!\Lss_{\SVMifn}f\Phi\,\,d[\iota(F_{n})] -
                 \!\!\int\!\!\Lss_{\SVMifz}f\Phi\,\,d[\iota(F_{n})]\,
           \bigg\|_{H}\!\!+\\
       & &+\;
           \bigg\|\int\!\!\Lss_{\SVMifz}f\Phi\,\,d[\iota(F_{n})] -
                  \!\int\!\!\Lss_{\SVMifz}f\Phi\,\,d[\iota(F_{0})]
           \bigg\|_{H},
  \end{eqnarray*}
  this can be done by showing
  \begin{eqnarray}\label{lemma-prep-hadamard-p5}
    \lim_{n\rightarrow\infty}\!\sup_{f\in H \atop \|f\|_{H}\leq 1}\!\!
           \bigg\|\int\!\!\Lss_{\SVMifn}f\Phi\,\,d[\iota(F_{n})] -\!\!
                 \int\!\!\Lss_{\SVMifz}f\Phi\,\,d[\iota(F_{n})]\,
           \bigg\|_{H}\!
    =\,0
  \end{eqnarray}
  and
  \begin{eqnarray}\label{lemma-prep-hadamard-p6}
    \lim_{n\rightarrow\infty}\!\sup_{f\in H \atop \|f\|_{H}\leq 1}\!\!
    \bigg\|\int\!\!\Lss_{\SVMifz}f\Phi\,\,d[\iota(F_{n})] -\!\!
                  \int\!\!\Lss_{\SVMifz}f\Phi\,\,d[\iota(F_{0})]
           \bigg\|_{H}\!
    =\,0\,.
  \end{eqnarray}
  In order to prove (\ref{lemma-prep-hadamard-p5}), define
  $$\tilde{F}_{n}\,:=\,\frac{1}{\iota(F_{n})\big(\XY\big)}F_{n}
    \quad\text{and}\quad
    \tilde{\lambda}_{n}\,:=\,\frac{\lambda_{0}}{\iota(F_{n})\big(\XY\big)}
    \quad\;\forall\,n\in\N\cup\{0\}\,,
  $$ 
  and
  $$\tilde{F}_{0,n}\,:=\,
    \frac{\lambda_{0}}{\tilde{\lambda}_{n}}\tilde{F}_{0}
    \;=\;\frac{\iota(F_{n})\big(\XY\big)}{\iota(F_{0})\big(\XY\big)}\,F_{0}
    \qquad\forall\,n\in\N\cup\{0\}
  $$   
  Then, $\iota(\tilde{F}_{n})$ is a probability measure and,
  according to Lemma \ref{lemma-convergence-in-BS-weak-convergence}, 
  it follows
  that 
  $\lim_{n\rightarrow\infty}\iota(F_{n})\big(\XY\big)
   =\iota(F_{0})\big(\XY\big)
  $
  and, therefore, 
  $\lim_{n\rightarrow\infty}\|\tilde{F}_{n}-\tilde{F}_{0}\|_{\infty}=0$\,.
  Hence,
  \begin{eqnarray}\label{lemma-prep-hadamard-p10-0}
    \lefteqn{
      \lim_{n\rightarrow\infty}\big\|\SVMifn-\SVMifz\big\|_{H}
      \;\stackrel{(\ref{prep-standard-regularization-parameter})}{=}\;
         \lim_{n\rightarrow\infty}
             \big\|f_{L,\iota(\tilde{F}_{n}),\tilde{\lambda}_{n}}
                   -f_{L,\iota(\tilde{F}_{0}),\tilde{\lambda}_{0}}\big\|_{H}
      \;\leq\;}  \nonumber \\
    &&\leq\;
         \lim_{n\rightarrow\infty}\,\,
             \big\|f_{L,\iota(\tilde{F}_{n}),\tilde{\lambda}_{n}}
                   -f_{L,\iota(\tilde{F}_{0}),\tilde{\lambda}_{n}}\big\|_{H}
             \,+\,
             \big\|f_{L,\iota(\tilde{F}_{0}),\tilde{\lambda}_{n}}
                   -f_{L,\iota(\tilde{F}_{0}),\tilde{\lambda}_{0}}\big\|_{H}
            \nonumber\\
    &&\stackrel{(\ast)}{\leq}\;
         \lim_{n\rightarrow\infty}
              \frac{1}{\tilde{\lambda}_{n}}
              \bigg\|\int\Ls_{f_{L,\iota(\tilde{F}_{0}),\tilde{\lambda}_{n}}}
                        \!\Phi\,d[\iota(\tilde{F}_{n}] -
                    \int\Ls_{f_{L,\iota(\tilde{F}_{0}),\tilde{\lambda}_{n}}}
                        \!\Phi\,d[\iota(\tilde{F}_{0}]
              \bigg\|_{H}  \nonumber  \\
    &&\stackrel{(\ref{prep-standard-regularization-parameter})}{=}\;
         \lim_{n\rightarrow\infty}
              \frac{1}{\tilde{\lambda}_{n}}
              \bigg\|\int\Ls_{f_{L,\iota(\tilde{F}_{0,n}),\lambda_{0}}}
                        \!\Phi\,d[\iota(\tilde{F}_{n}] -
                    \int\Ls_{f_{L,\iota(\tilde{F}_{0,n}),\lambda_{0}}}
                        \!\Phi\,d[\iota(\tilde{F}_{0}]
              \bigg\|_{H}  \nonumber  \\
    &&=\;\lim_{n\rightarrow\infty}\frac{1}{\tilde{\lambda}_{n}}
        \big\|W_{\tilde{F}_{0,n}}(\tilde{F}_{n})
               -W_{\tilde{F}_{0,n}}(\tilde{F}_{0})
        \big\|_{H}
  \end{eqnarray}
  where $(\ast)$ follows from 
  \cite[Theorem 5.9 and Corollary 5.19]{steinwart2008}.\\  
  Since 
  $\lim_{n\rightarrow\infty}\iota(F_{n})\big(\XY\big)
   =\iota(F_{0})\big(\XY\big)
  $, 
  it follows from
  (\ref{prep-def-c0-constant}), (\ref{prep-inequality-hnorm})
  and (\ref{lemma-prep-hadamard-0})
  that
  $$\big\|f_{\iota(\tilde{F}_{0,n})}\big\|_{H}\;\leq\;c_{0}
    \qquad\text{for large enough }n\in\N.
  $$
  Hence, 
  \begin{eqnarray}\label{lemma-prep-hadamard-p10}
    \lefteqn{
      \lim_{n\rightarrow\infty}\big\|\SVMifn-\SVMifz\big\|_{H}
        \;\stackrel{(\ref{lemma-prep-hadamard-p10-0})}{\leq}\;
          \lim_{n\rightarrow\infty}\frac{1}{\tilde{\lambda}_{n}}
           \big\|W_{\tilde{F}_{0,n}}(\tilde{F}_{n})
                  -W_{\tilde{F}_{0,n}}(\tilde{F}_{0})
           \big\|_{H} 
     \;\leq\;}  \nonumber \\
     & & \qquad\stackrel{(\ref{lemma-prep-hadamard-p3})}{\leq}\;
           \lim_{n\rightarrow\infty}
               \frac{1}{\tilde{\lambda}_{n}}
               \big\|\tilde{F}_{n}-\tilde{F}_{0}\big\|_{\infty}
       \;=\;0  \qquad \qquad\qquad\qquad\qquad\qquad
  \end{eqnarray}
  Therefore, (\ref{lemma-prep-hadamard-p5}) follows from
  Lemma \ref{lemma-continuous-partial-derivative}.
  
  In order to prove (\ref{lemma-prep-hadamard-p6}), 
  define $M:=\sup_{n\in\N\cup\{0\}}\iota(F_{n})\big(\XY\big)<\infty$
  (see Lemma \ref{lemma-convergence-in-BS-weak-convergence}) and
  note that, according to \cite[Corollary 4.31]{steinwart2008},
  $$\F_{1}\;=\;
    \big\{f\in H\;\big|\;\;\|f\|_{H}\leq 1\big\}
    \;\subset\;\mathcal{C}(\X)
  $$
  can be identified with 
  a relatively compact subset of $\mathcal{C}(\X)$
  (with respect to the norm-topology of $\mathcal{C}(\X)$)\,.
  Hence, for every $\varepsilon>0$,
  there is an $m_{\varepsilon}\in\N$ and
  functions 
  $\,f_{1},\dots,f_{m_{\varepsilon}}\,\in\,\mathcal{C}(\X)\,$
  such that 
  \begin{eqnarray}
   \label{lemma-prep-hadamard-p7}
    \|f_{j}\|_{\infty}\;\leq\;
    \sup_{f\in \F_{1}}\|f\|_{\infty}
    \;\stackrel{(\ref{hilbert-norm-uniform-norm})}{\leq}\;
    \|k\|_{\infty}
    \qquad\forall\,j\in\{1,\dots,m_{\varepsilon}\}, \\
   \label{lemma-prep-hadamard-p8}
    \min_{j\in\{1,\dots,m_{\varepsilon}\}}
    \big\|f-f_{j}\big\|_{\infty}\;<\;\varepsilon
    \qquad\forall\,f\in\F_{1}
    \;.
  \end{eqnarray}
  Define $a:=\|\SVMifz\|_{\infty}$\,.
  Fix any $f\in \F_{1}$ and take $j_{0}\in\{1,\dots,m_{\varepsilon}\}$
  such that $\|f-f_{j_{0}}\|_{\infty}<\varepsilon$\,. 
  Then,
  \begin{eqnarray*}
    \lefteqn{
       \left\|
          \int \Lss_{\SVMifz}f\Phi \,d\big[\iota(F_{n})\big] -
          \int \Lss_{\SVMifz}f\Phi \,d\big[\iota(F_{0})\big]
       \right\|_{H}\;=}\\
    &=&\left\|
          \int \Lss_{\SVMifz}(f-f_{j_{0}})\Phi \,d\big[\iota(F_{n})\big] -
          \int \Lss_{\SVMifz}(f-f_{j_{0}})\Phi \,d\big[\iota(F_{0})\big]
        \,-\right. \\
       & &\quad -\; \left.
          \int \Lss_{\SVMifz}f_{j_{0}}\Phi \,d\big[\iota(F_{0})\big] +
          \int \Lss_{\SVMifz}f_{j_{0}}\Phi \,d\big[\iota(F_{n})\big]
       \right\|_{H} \;\leq\\
    &\leq&
          \int \!\big\|\Lss_{\SVMifz}(f-f_{j_{0}})\Phi\big\|_{H} 
          \,d\big[\iota(F_{n})\big] 
          +
          \int \big\|\Lss_{\SVMifz}(f-f_{j_{0}})\Phi\big\|_{H} 
          \,d\big[\iota(F_{0})\big] 
           \\
       & &\quad +\; 
          \left\|
             \int \Lss_{\SVMifz}f_{j_{0}}\Phi \,d\big[\iota(F_{0})\big] -
             \int \Lss_{\SVMifz}f_{j_{0}}\Phi \,d\big[\iota(F_{n})\big]
          \right\|_{H} \;\leq\\
    &\stackrel{(\ref{theorem-sqrt-n-consistency-1},
                \ref{bounded-feature-map})}{\leq}\!\!&
        2b_{a}^{\prime\prime}\|k\|_{\infty}M\varepsilon\,+\,
          \left\|
             \int \Lss_{\SVMifz}f_{j_{0}}\Phi \,d\big[\iota(F_{0})\big] -
             \int \Lss_{\SVMifz}f_{j_{0}}\Phi \,d\big[\iota(F_{n})\big]
          \right\|_{H}        
  \end{eqnarray*}
  Hence, 
  \begin{eqnarray}\label{lemma-prep-hadamard-p9}
    \lefteqn{\sup_{f\in \F_{1}}
       \left\|
          \int \Lss_{\SVMifz}f\Phi \,d\big[\iota(F_{n})\big] -
          \int \Lss_{\SVMifz}f\Phi \,d\big[\iota(F_{0})\big]
       \right\|_{H}\;\leq}\\
    &\leq& 
        2b_{a}^{\prime\prime}\|k\|_{\infty}M\varepsilon\,+\!\!\!
        \max_{j\in\{1,\dots,m_{\varepsilon}\}}\!
          \left\|\!
             \int \!\! \Lss_{\SVMifz}f_{j}\Phi \,d\big[\iota(F_{0})\big] 
             \! - \!\!
             \!\int \!\! \Lss_{\SVMifz}f_{j}\Phi \,d\big[\iota(F_{n})\big]\!
          \right\|_{H} \!\! .  \nonumber     
  \end{eqnarray}
  Convergence of $(F_{n})_{n\in\N}$ in $\ell_{\infty}(\G)$ implies
  weak convergence (Lemma \ref{lemma-convergence-in-BS-weak-convergence})
  and, therefore, tightness of the
  sequence of finite measures $(\iota(F_{n}))_{n\in\N}$;
  see e.g.\ \cite[Theorem 30.8]{bauer2001}.
  Hence, there is a compact set $\Z_{\varepsilon}\subset\XY$
  such that, for its complement $\complement\Z_{\varepsilon}$, we have
  $\,\sup_{n\in\N_{0}}
   \iota(F_{n})\big(\complement\Z_{\varepsilon}\big)<\varepsilon
  $. 
  Then,
  \begin{eqnarray*}
    \lefteqn{ 
      \max_{j\in\{1,\dots,m_{\varepsilon}\}}
          \left\|
             \int \Lss_{\SVMifz}f_{j}\Phi \,d\big[\iota(F_{0})\big] -
             \int \Lss_{\SVMifz}f_{j}\Phi \,d\big[\iota(F_{n})\big]
          \right\|_{H} \;\leq}\\
    &\leq&\!\!\!\!\max_{j\in\{1,\dots,m_{\varepsilon}\}}\!
          \left\|
             \int_{\Z_{\varepsilon}} \Lss_{\SVMifz}f_{j}\Phi 
             \,d\big[\iota(F_{0})\big] -
             \int_{\Z_{\varepsilon}} \Lss_{\SVMifz}f_{j}\Phi 
             \,d\big[\iota(F_{n})\big]
          \right\|_{H}\;+  \\
       & &\qquad\;\;\;\; + \,  
          \left\|
             \int_{\complement\Z_{\varepsilon}} \Lss_{\SVMifz}f_{j}\Phi 
             \,d\big[\iota(F_{0})\big] 
          \right\|_{H}\!\!+\,
          \left\|
             \int_{\complement\Z_{\varepsilon}} \Lss_{\SVMifz}f_{j}\Phi 
             \,d\big[\iota(F_{n})\big]
          \right\|_{H} \\
    &\stackrel{(\ref{theorem-sqrt-n-consistency-1},
              \ref{lemma-prep-hadamard-p7})}{\leq}&\!\!\!\!
		  \max_{j\in\{1,\dots,m_{\varepsilon}\}}\!
          \left\|
             \int_{\Z_{\varepsilon}}\!\!\!\! \Lss_{\SVMifz}f_{j}\Phi 
             \,d\big[\iota(F_{0})\big]\! -\!\!
             \int_{\Z_{\varepsilon}}\!\!\! \Lss_{\SVMifz}f_{j}\Phi 
             \,d\big[\iota(F_{n})\big]
          \right\|_{H}\!\!\!+   
             2b_{a}^{\prime\prime}\|k\|_{\infty}^{2}\varepsilon.
  \end{eqnarray*}
  According to \cite[p.\ III.40]{bourbaki2004integration}, 
  weak convergence of the
  sequence of finite (positive) measures $(\iota(F_{n}))_{n\in\N}$
  implies
  $$\lim_{n\rightarrow\infty}
        \left\|
             \int_{\Z_{\varepsilon}} \Lss_{\SVMifz}f_{j}\Phi 
             \,d\big[\iota(F_{0})\big] -
             \int_{\Z_{\varepsilon}} \Lss_{\SVMifz}f_{j}\Phi 
             \,d\big[\iota(F_{n})\big]
        \right\|_{H} 
    \;=\;0 \quad
  $$ 
  for every $j\in\{1,\dots,m_{\varepsilon}\}$\,.
  (Since $H$ is a separable Banach space,
  Pettis integrals and Bochner-integrals coincide; see e.g.\ 
  \cite[p.\ 194f]{dudley2002}.) 
  As $\varepsilon>0$ can be arbitrarily small, 
  (\ref{lemma-prep-hadamard-p6}) follows from
  (\ref{lemma-prep-hadamard-p9}) and the above calculation.

 \textit{Part 3:}
  In this part of the proof, it will be shown that
  \begin{eqnarray}\label{lemma-prep-hadamard-p11}
    \lim_{n\rightarrow\infty}
      \big\|W_{F_{n}}(G_{0})-W_{F_{0}}(G_{0})\big\|_{H}
    \;\;=\;\;0\;.
  \end{eqnarray}
  For every $m\in\N$\,, we have $G_{m}\in\text{lin}(B_{S})$ and, 
  therefore,
  $$W_{F_{n}}(G_{m})=
    \int \Ls_{\SVMifn}\Phi \,d\big[\iota(G_{m})\big]
  $$ for every $n\in\N_{0}$.
  Hence, it follows from 
  (\ref{lemma-prep-hadamard-p10}) and
  Lemma \ref{lemma-derivatives}\,b) that
  \begin{eqnarray}\label{lemma-prep-hadamard-p12}
    \lim_{n\rightarrow\infty}
      \big\|W_{F_{n}}(G_{m})-W_{F_{0}}(G_{m})\big\|_{H}
    \;\;=\;\;0
    \qquad\forall\,m\in\N\;.
  \end{eqnarray}
  Furthermore, we have 
  \begin{eqnarray}\label{lemma-prep-hadamard-p13}
    \lim_{m\rightarrow\infty}\sup_{n\in\N_{0}}\!\!
      \big\|W_{F_{n}}(G_{m})-W_{F_{n}}(G_{0})\big\|_{H}
      \!\stackrel{(\ref{lemma-prep-hadamard-p3})}{\leq}\!
          \lim_{m\rightarrow\infty}\!\big\|G_{m}-G_{0}\big\|_{\infty}
    \!\!=\,0
  \end{eqnarray}
  According to \cite[I.7.6]{dunford1958}, 
  (\ref{lemma-prep-hadamard-p12}) and (\ref{lemma-prep-hadamard-p13})
  imply
  \begin{eqnarray*}
    \lefteqn{
      \lim_{n\rightarrow\infty}
        \big\|W_{F_{n}}(G_{0})-W_{F_{0}}(G_{0})\big\|_{H}
      \;=\;\lim_{n\rightarrow\infty}\lim_{m\rightarrow\infty}
           \big\|W_{F_{n}}(G_{m})-W_{F_{0}}(G_{m})\big\|_{H}=
    }\\
    &&=\;\lim_{m\rightarrow\infty}\lim_{n\rightarrow\infty}
           \big\|W_{F_{n}}(G_{m})-W_{F_{0}}(G_{m})\big\|_{H}
       \;=\;0\;. \qquad\qquad\qquad\qquad\quad\;
  \end{eqnarray*}

 \textit{Part 4:} By use of the
  previous parts, we complete the proof by
  proving (\ref{lemma-prep-hadamard-1}): 
  \begin{eqnarray*}
    \lefteqn{
      \lim_{n\rightarrow\infty}
      \big\|S_{F_{n}}^{\prime}(G_{n})-\!S_{F_{0}}^{\prime}(G_{0})\big\|_{H}
      =
     \lim_{n\rightarrow\infty}\!
           \big\|K_{F_{n}}^{-1}\big(W_{F_{n}}(G_{n})\big)
                 \!-\!K_{F_{0}}^{-1}\big(W_{F_{0}}(G_{0})\big)
           \big\|_{H}}\\
    &\leq&\lim_{n\rightarrow\infty}\;
           \big\|K_{F_{n}}^{-1}\big(W_{F_{n}}(G_{n})\big)
                 -K_{F_{0}}^{-1}\big(W_{F_{n}}(G_{n})\big)
           \big\|_{H}+ \\
        &&\qquad\quad +\,\,
           \big\|K_{F_{0}}^{-1}\big(W_{F_{n}}(G_{n})\big)
                 -K_{F_{0}}^{-1}\big(W_{F_{n}}(G_{0})\big)
           \big\|_{H}+ \\
        &&\qquad\quad +\,\,
           \big\|K_{F_{0}}^{-1}\big(W_{F_{n}}(G_{0})\big)
                 -K_{F_{0}}^{-1}\big(W_{F_{0}}(G_{0})\big)
           \big\|_{H}\;= \\ 
    &\stackrel{(\ref{lemma-prep-hadamard-p11})}{\leq}&
          \lim_{n\rightarrow\infty}\;
           \big\|K_{F_{n}}^{-1}\!-\!K_{F_{0}}^{-1}\big\|\cdot\!
           \big\|W_{F_{n}}(G_{n})\big\|_{H}+\!
           \big\|K_{F_{0}}^{-1}\big\|\cdot\!       
           \big\|W_{F_{n}}(G_{n})\!-\!W_{F_{n}}(G_{0})
           \big\|_{H} \\       
    &\stackrel{(\ref{lemma-prep-hadamard-p3})}{\leq}&
          \lim_{n\rightarrow\infty}\;\;
           \big\|K_{F_{n}}^{-1}-K_{F_{0}}^{-1}\big\|\cdot
           \big\|G_{n}\big\|_{\infty} \,+\, 
           \big\|K_{F_{0}}^{-1}\big\|\cdot       
           \big\|G_{n}-G_{0}\big\|_{\infty}\;=\;0 \qquad       
  \end{eqnarray*}
\end{proof}

\begin{theorem}\label{theorem-hadamard}
  For every $F_{0}\in B_{S}$ which fulfills
  (\ref{lemma-prep-hadamard-0}),
  the map 
  $$S\;:\;\;B_{S}\;\rightarrow\;H\,,\qquad
    F\;\mapsto\;f_{\iota(F)}
  $$
  is Hadamard-differentiable in $F_{0}$ tangentially to
  the closed linear span
  $B_{0}=\textup{cl}\big(\textup{lin}(B_{S})\big)$\,. 
  The derivative in $F_{0}$ is a continuous linear
  operator $\,S_{F_{0}}^{\prime}:B_{0}\rightarrow H\,$ such that
  \begin{eqnarray}\label{lemma-hadamard-1}
    S_{F_{0}}^{\prime}(G)\;=\;
        -K_{F_0}^{-1}\Big(\E_{\iota(G)}\big(\Ls_{\SVMifz}\Phi\big)\Big)
    \qquad\quad\forall\,G\in\textup{lin}(B_{S})\;.
  \end{eqnarray}
\end{theorem}
\begin{proof}\item
  Let $(G_{n})_{n\in\N}\subset\ell_{\infty}(\G)$
  and $(t_{n})_{n\in\N}\subset\R\setminus\{0\}$ be sequences
  such that $\,\lim_{n\rightarrow\infty}\|G_{n}-G_{0}\|_{\infty}=0\,$
  for some $G_{0}\in\ell_{\infty}(\G)$, such that
  $\,t_{n}\searrow 0$, and such that
  $\,F_{n}:=F_{0}+t_{n}G_{n}\,\in\,B_{S}\,$
  for every $n\in\N$\,.   
  Then, $\lim_{n\rightarrow\infty}\|F_{n}-F_{0}\|_{\infty}=0$ and
  $G_{n}\in\textup{lin}(B_{s})$ for every $n\in\N$.
  According to Lemma \ref{lemma-prep-hadamard},
  there is a $n_{0}\in\N$ such that, for every 
  $F\in\{F_{n}|n\in\N_{\geq n_{0}}\}\cup\{F_{0}\}$, 
  there is a continuous linear operator
  $S_{F}^{\prime}:B_{0}\rightarrow H
  $
  which fulfills (\ref{lemma-hadamard-1}).
  We have to show
  \begin{eqnarray}\label{theorem-hadamard-p1}
    \lim_{n\rightarrow\infty}
    \left\|\frac{S(F_{0}+t_{n}G_{n})-S(F_{0})}{t_{n}}
           -S_{F_{0}}^{\prime}(G_{0})
    \right\|_{H}
    \;=\;0\;.
  \end{eqnarray} 
  Note that the assumptions imply $G_{0}\in B_{0}$\,.
  Define
  \begin{eqnarray}\label{theorem-hadamard-p1001}
    h_{n}\;:=\;S(F_{0}+t_{n}G_{n})-S(F_{0})-t_{n}S_{F_{0}}^{\prime}(G_{0})
    \qquad\forall\,n\in\N\,.
  \end{eqnarray}
  That is, for every $f\in H$\,,
  \begin{eqnarray}\label{theorem-hadamard-p2}
    \la f,h_{n}\ra\;=\;
    \la f,S(F_{0}+t_{n}G_{n})-S(F_{0}) \ra 
    - \la f,t_{n}S_{F_{0}}^{\prime}(G_{0}) \ra
    \;.
  \end{eqnarray} 
  In order to prove for every $n\in\N$ that the function
  $$[0,1]\;\rightarrow\;H\,,\qquad
    s\;\mapsto\;S(F_{0}+st_{n}G_{n})
  $$
  is well-defined, we have to show that
  $\,F_{0}+st_{n}G_{n}\,\in\,B_{S}\,$ for every
  $s\in[0,1]$\,. It follows from $F_{n}\in B_{S}$ that
  $G_{n}\in\textup{lin}(B_{S})$\,. Therefore,
  there is a finite signed measure $\mu_{n,s}$ such that
  $\mu_{n,s}=\iota(F_{0}+st_{n}G_{n})$ and
  $\,F_{0}+st_{n}G_{n}\,\in\,\textup{lin}(B_{S})\,$.
  Take any $A\in\mathfrak{B}(\XY)$\,. Then, it follows from
  $\iota(F_{0})(A)\geq 0$\,, $\iota(F_{n})(A)\geq 0$ and
  $s\in[0,1]$ that  
  $\mu_{n,s}(A)=\iota(F_{0}+st_{n}G_{n})(A)\geq0$\,.
  That is, $\mu_{n,s}=\iota(F_{0}+st_{n}G_{n})$ is a
  finite measure.   
  Furthermore, it follows from $F_{0}\not=0$,
  $F_{n}\not=0$ and $s\in[0,1]$ that
  $\mu_{n,s}\not=0$\,.
  According to the definitions, this shows that
  $\,F_{0}+st_{n}G_{n}\,\in\,B_{S}\,$.
  
  Fix any $n\in\N$.
  The function
  $s\mapsto  S(F_{0}+st_{n}G_{n})
  $
  is continuous on $[0,1]$
  according to (\ref{lemma-prep-hadamard-p10}) 
  and Frech\'{e}t-differentiable on $(0,1)$
  according to Proposition \ref{prop-gateaux}; the
  derivative in $s\in(0,1)$ is given by
  $S_{F_{0}+st_{n}G_{n}}^{\prime}(t_{n}G_{n})$\,.
  Since the map $h\mapsto\la f,h\ra$ is
  Frech\'{e}t-differentiable  
  for every $f\in H$, 
  this implies that
  $$(0,1)\;\rightarrow\;\R\,,\qquad
    s\;\mapsto\;\la f,S(F_{0}+st_{n}G_{n}) \ra
  $$
  is differentiable for every $f\in H$; the
  derivative in $s\in(0,1)$ is given by
  $\la f,S_{F_{0}+st_{n}G_{n}}^{\prime}(t_{n}G_{n}) \ra$\,. 
  Define $\tilde{h}_{n}=h_{n}/\|h_{n}\|_{H}$\,.
  According to the 
  elementary mean value theorem,
  there is an $\tilde{s}_{n}\in(0,1)$ such that
  \begin{eqnarray*}
    \big\la \tilde{h}_{n},
            S_{F_{0}+\tilde{s}_{n}t_{n}G_{n}}^{\prime}(t_{n}G_{n}) 
    \big\ra
    &=& \big\la \tilde{h}_{n},S(F_{0}+t_{n}G_{n}) \big\ra - 
        \big\la \tilde{h}_{n},S(F_{0}) \big\ra \;=\\
    &=& \big\la \tilde{h}_{n},S(F_{0}+t_{n}G_{n})- S(F_{0})\big\ra
  \end{eqnarray*}
  By use of the definition of $h_{n}$\,, this implies
  $$\big\la \tilde{h}_{n},h_{n} \big\ra \;=\;
    \big\la \tilde{h}_{n}\,,\,
            S_{F_{0}+\tilde{s}_{n}t_{n}G_{n}}^{\prime}(t_{n}G_{n})
            -t_{n}S_{F_{0}}^{\prime}(G_{0})
    \big\ra
  $$
  and, by use of the definition of $\tilde{h}_{n}$\,,
  the latter equality and the Cauchy-Schwarz inequality
  imply
  \begin{eqnarray}\label{theorem-hadamard-p3}
    \big\|h_{n}\|_{H}\;\leq\;
    \big\|S_{F_{0}+\tilde{s}_{n}t_{n}G_{n}}^{\prime}(t_{n}G_{n})
            -t_{n}S_{F_{0}}^{\prime}(G_{0})
    \big\|_{H}\;.
  \end{eqnarray}
  Then, (\ref{theorem-hadamard-p1}) follows from
  \begin{eqnarray*}
    \lefteqn{
       \left\|\frac{S(F_{0}+t_{n}G_{n})-S(F_{0})}{t_{n}}
              -S_{F_{0}}^{\prime}(G_{0})
       \right\|_{H}\;=}\\
    &\!=&\frac{\big\|S(F_{0}+t_{n}G_{n})
                   -S(F_{0})-t_{n}S_{F_{0}}^{\prime}(G_{0})
             \big\| %
            }{t_{n}}
       \;\,\,\stackrel{(\ref{theorem-hadamard-p1001})}{=}\;\,\,
       \frac{1}{t_{n}}\big\|h_{n}\big\|_{H}\;\leq\\
    &\!\stackrel{(\ref{theorem-hadamard-p3})}{\leq}&
       \frac{1}{t_{n}}
       \big\|S_{F_{0}+\tilde{s}_{n}t_{n}G_{n}}^{\prime}(t_{n}G_{n})
               \!-t_{n}S_{F_{0}}^{\prime}(G_{0})
       \big\|_{H}\!=\!
          \big\|S_{F_{0}+\tilde{s}_{n}t_{n}G_{n}}^{\prime}(G_{n})
                  -S_{F_{0}}^{\prime}(G_{0})
          \big\|_{H}
  \end{eqnarray*} 
  because the last expression converges to 0 according
  to Lemma \ref{lemma-prep-hadamard}.
\end{proof}

\subsection{Donsker-Classes and Application of the Delta-Method}
  \label{subsec-donsker-delta}

It is well-known that
$$\sqrt{n}\big(\mathds{F}_{n}-F\big)
    \;\leadsto\;
    \mathds{G}_{1}
    \qquad\text{in}\quad\ell_{\infty}(\G_{1})
$$
where $\mathds{F}_{n}$ denotes the empirical process,
$F$ denotes the distribution function of $P$,
$\mathds{G}_{1}$ is a Gaussian process, and 
$\mathcal{G}_{1}$
is the set of all indicator functions. 
However, as already noted in Subsection \ref{subsec-preparation-for-proof},
the set of indicator functions 
had to be enlarged to a set $\G\supset\G_{1}$ in order to
ensure Hadamard-differentiability of the SVM-functional
$$S\;:\;\;B_{S}\;\longrightarrow\;H
$$
in a neighborhood of $F\in B_{S}\subset\ell_{\infty}(\G)$. 
Therefore, it still has to be proven that
weak convergence not only holds in $\ell_{\infty}(\G_{1})$
but also in $\ell_{\infty}(\G)$.
This is done in the following Lemma \ref{lemma-donsker}.
After that, the main results can be proven by 
applications of a functional delta-method.

\begin{lemma}\label{lemma-donsker}
  For every 
  $D_{n}=\big((x_{1},y_{1}),\dots,(x_{n},y_{n})\big)\in(\XY)^{n}$, 
  let
  $\mathds{F}_{D_{n}}$ denote the element of
  $\ell_{\infty}(\G)$ which corresponds to the
  empirical measure $\mathds{P}_{D_{n}}$\,. That is,
  $\mathds{F}_{D_{n}}(g)=\int g \,d\mathds{P}_{D_{n}}=
    \frac{1}{n}\sum_{i=1}^{n}g(x_{i},y_{i})
  $  for every $g\in\G$\,.\\
  Then,
  $$\sqrt{n}\big(\mathds{F}_{\mathbf{D}_{n}}-\iota^{-1}(P)\big)
    \;\leadsto\;
    \mathds{G}
    \qquad\text{in}\quad\ell_{\infty}(\G)
  $$
  where $\mathds{G}:\Omega\rightarrow\ell_{\infty}(\G)$
  is a tight Borel-measurable Gaussian process such that
  $\mathds{G}(\omega)\in B_{0}$
  for every $\omega\in\Omega$.
\end{lemma}
\begin{proof}\item
  In other words, we have to show that 
  $\G$ is a $P$\,-\,Donsker class.
  
 \textit{Part 1:} 
   Fix any $c\in(0,\infty)$\,. 
   In \textit{Part 1} of the proof, it will be shown that
   $$\F_{c}\;:=\;
     \big\{f\in H\,\big|\,\,\|f\|_{H}\leq c\big\}
   $$
   has a finite uniform entropy integral.
   Since $\X\subset\R^d$ is bounded, there is an $r>0$ such that
   $\X\subset\big\{x\in\R^d\,\big|\,\,\|x\|_{\R^d}<r\big\}
     \,=:\,\tilde{\X}\,.
   $
   Then, $\tilde{\X}$ is a convex, bounded subset of $\R^d$ with
   non-empty interior.
   Let $\tilde{H}$ be the RKHS of
   the restriction of 
   the kernel $\tilde{k}$ on $\tilde{\X}\times\tilde{\X}$
   and define
   $$\tilde{\F}_{c}\;:=\;
     \big\{\tilde{f}\in\tilde{H}\,
     \big|\,\,\|\tilde{f}\|_{\tilde{H}}\leq c\big\}\;.
   $$
   It follows from 
   \cite[Theorem 4.2.6]{berlinet2004}
   that
   \begin{eqnarray}\label{lemma-donsker-p1}
     \F_{c}\;:=\;
     \big\{f\in H\,
     \big|\,\,f\,\,\text{is the restriction of some}\,\,
              \tilde{f}\in\tilde{H}
     \big\}\;.
   \end{eqnarray}
   According to 
   \cite[p.\ 154]{vandervaartwellner1996},
   let $\mathcal{C}_{1}^{m}(\tilde{\X})$ denote the
   set of all functions $\tilde{f}:\tilde{\X}\rightarrow\R$
   which have uniformly bounded partial derivatives up to
   order $m-1$ and whose partial derivatives of order
   $m-1$ are Lipschitz-continuous such that
   $$\big\|\tilde{f}\big\|_{1}\;:=\;
     \max_{\alpha\in\N_{0} \atop |\alpha|\leq m-1}\sup_{x\in\tilde{\X}}
        \big|\partial^{\alpha}\tilde{f}(x)\big|\,+\!
     \max_{\alpha\in\N_{0} \atop |\alpha|= m-1}
        \sup_{x,x^{\prime}\in\tilde{\X} \atop x\not= x^{\prime}}
        \frac{\big|\partial^{\alpha}\tilde{f}(x)
                   -\partial^{\alpha}\tilde{f}(x^{\prime})
              \big| %
             }{\|x-x^{\prime}\|_{\R^{d}}}
     \;\,\,\leq\,\,\;1\;.
   $$
   It follows from convexity of $\tilde{\X}$ and the
   mean value theorem 
   that
   $$\max_{\alpha\in\N_{0} \atop |\alpha|= m-1}
        \sup_{x,x^{\prime}\in\tilde{\X} \atop x\not= x^{\prime}}
        \frac{\big|\partial^{\alpha}\tilde{f}(x)
                   -\partial^{\alpha}\tilde{f}(x^{\prime})
              \big| %
             }{\|x-x^{\prime}\|_{\R^{d}}}
     \;\,\,\leq\,\,\;
     \max_{\alpha\in\N_{0} \atop |\alpha|= m}\sup_{x\in\tilde{\X}}
            \big|\partial^{\alpha}\tilde{f}(x)\big|
     \;\;.
   $$
   Hence, it follows from \cite[Corollary 4.36]{steinwart2008}
   that, for every $\tilde{f}\in\tilde{\F}_{c}$\,,
   \begin{eqnarray*}
     \big\|\tilde{f}\big\|_{1}
     &\leq& \max_{\alpha\in\N_{0} \atop |\alpha|\leq m}\sup_{x\in\tilde{\X}}
            \big|\partial^{\alpha}\tilde{f}(x)\big|
        \;\,\,\leq\,\,\; \big\|\tilde{f}\big\|_{\tilde{H}}
          \max_{\alpha\in\N_{0} \atop |\alpha|\leq m}\sup_{x\in\tilde{\X}}
            \big(\partial^{\alpha,\alpha}\tilde{k}(x,x)\big)\\
     &\leq& c\cdot 
          \max_{\alpha\in\N_{0} \atop |\alpha|\leq m}\sup_{x\in\tilde{\X}}
            \big(\partial^{\alpha,\alpha}\tilde{k}(x,x)\big)
        \;\,\,=:\,\,\;a_{c}\;\in\;(0,\infty)\;.
   \end{eqnarray*}
   That is,
   $\frac{1}{a_{c}}\tilde{\F}_{c}\subset\mathcal{C}_{1}^{m}(\tilde{\X})$
   and, therefore, it follows from 
   \cite[Theorem 2.7.1]{vandervaartwellner1996}
   that there is a constant $r\in(0,\infty)$ 
   such that, for every $\varepsilon>0$\,,
   \begin{eqnarray}\label{lemma-donsker-p4}
     \ln N\big(a_{c}\varepsilon,\tilde{\F}_{c},\|\cdot\|_{\infty}
               \big)
     \,=\,
     \ln N\big(\varepsilon,{\textstyle\frac{1}{a_{c}}}\tilde{\F}_{c},
               \|\cdot\|_{\infty}
          \big) 
     \,\leq\, r\cdot\left(\frac{1}{\varepsilon}\right)^{\frac{d}{m}}\,.
   \end{eqnarray}
   Here and in the following, $N(\cdot,\cdot,\cdot)$
   denotes the covering number and $N_{[\,]}(\cdot,\cdot,\cdot)$
   denotes the bracketing number; see e.g.\ 
   \cite[\S\,2.1.1]{vandervaartwellner1996}. 
   According to (\ref{lemma-donsker-p1}),
   $\F_{c}$ is the set of restrictions of the elements
   of $\tilde{\F}_{c}$ on $\X$\,.
   By use of this fact, it is easy to see that
   $$\ln N\big(\varepsilon,\F_{c},\|\cdot\|_{\infty}
          \big)
     \;\leq\;
     \ln N\big(\varepsilon,\tilde{\F}_{c},\|\cdot\|_{\infty}
          \big)
   $$
   for every $\varepsilon>0$\,. 
   Therefore, it follows from (\ref{lemma-donsker-p4}) that
   \begin{eqnarray}\label{lemma-donsker-p2}
     \ln N\big(\varepsilon,\F_{c},\|\cdot\|_{\infty}
          \big)
     \;\leq\;r\cdot
             a_{c}^{\frac{d}{m}}
             \left(\frac{1}{\varepsilon}\right)^{\frac{d}{m}}
     \qquad\forall\,\varepsilon>0\;.\quad
   \end{eqnarray}
   Now, choose the constant 
   $f_{c}=\|k\|_{\infty}c+1$
   as an envelope of $\F_{c}$\,. 
   Every element $f\in\F_{c}$ can be identified with
   a function $\XY\rightarrow\R$ via $f(x,y)=f(x)$\,.
   For every probability measure
   $\tilde{P}$ on $(\XY,\mathfrak{B}(\XY))$\,, we obtain
   $$\|f\|_{L_{2}(\tilde{P})}
     \;\leq\;\sup_{(x,y)\in\XY}\big|f(x,y)\big|
     \;=\;\sup_{x\in\X}\big|f(x)\big|
     \;=\;\|f\|_{\infty}\;.
   $$
   Therefore, it follows from
   (\ref{lemma-donsker-p2}) that
   \begin{eqnarray}\label{bound-on-uniform-entropy-number}
     \sup_{\tilde{P}}\ln 
         N\big(\varepsilon\|f_{c}\|_{L_{2}(\tilde{P})},
               \F_{c},\|\cdot\|_{L_{2}(\tilde{P})}
          \big)
     \;\leq\;r
             \left(\frac{a_{c}}{\|k\|_{\infty}c+1}\right)^{\frac{d}{m}}
             \left(\frac{1}{\varepsilon}\right)^{\frac{d}{m}}
     \;\;
   \end{eqnarray}
   where the supremum is taken over all 
   probability measures $\tilde{P}$ on $(\XY,\mathfrak{B}(\XY))$\,.
   Since $m>\frac{d}{2}$ by assumption, the function class 
   $\F_{c}$ has a finite uniform entropy integral.
   That is,
   $$\int_{(0,1)}
     \sqrt{\sup_{\tilde{P}}\ln 
               N\big(\varepsilon\|f_{c}\|_{L_{2}(\tilde{P})},\F_{c},
                     \|\cdot\|_{L_{2}(\tilde{P})}
                \big)}
     \,\lambda(d\varepsilon)
     \;<\;\infty\;.
   $$

 \textit{Part 2:} Now, it will be shown that
   $$\G^{\prime}\;:=\;
     \Big\{\Ls_{f}:(x,y)\mapsto\Ls(x,y,f(x))\,\,\Big|\; f\in\F_{c_{0}}\Big\}
   $$
   also has a finite uniform entropy integral.   
   Since   
   $$\sup_{x\in\X}|f(x)|
     \;\stackrel{(\ref{hilbert-norm-uniform-norm})}{\leq}\;
     \|k\|_{\infty}c_{0}\;=:\;a
     \qquad\forall\,f\in\F_{c_{0}}\,,
   $$
   the assumptions imply that
   $g^{\prime}:=b_{a}^{\prime\prime}+b_{a}^{\prime}$
   is an envelope function of $\G^{\prime}$
   such that $0\leq b_{a}^{\prime\prime}\leq g^{\prime}$
   and, for every $(x,y)\in\XY$
   and every $f_{1},f_{2}\,\in\,\F_{c_{0}}$\,,
   \begin{eqnarray}\label{lemma-donsker-p201}
     \big|\Ls_{f_{1}}(x,y)\!-\!\Ls_{f_{2}}(x,y)\big|
     \stackrel{(\ast)}{\leq}
     b_{a}^{\prime\prime}\big|f_{1}(x)-f_{2}(x)\big|
     \leq g^{\prime}(x,y)\big\|f_{1}\!-\!f_{2}\big\|_{\infty}
   \end{eqnarray}
   where $(\ast)$ follows from the assumptions on $\Lss$ and
   the elementary mean value theorem.
   For every probability measure $\tilde{P}$ on $(\XY,\mathfrak{B}(\XY))$
   such that $\,0\,<\,\int(g^{\prime})^2 \,d\tilde{P}\,<\,\infty\,$\,,
   it follows from (\ref{lemma-donsker-p201}) and
   \cite[p.\ 84 and Theorem 2.7.11]{vandervaartwellner1996}   
   that, for every $\varepsilon>0$,
   \begin{eqnarray*}
     \lefteqn{
        \ln N\big(\varepsilon\|g^{\prime}\|_{L_{2}(\tilde{P})},
                         \G^{\prime},\|\cdot\|_{L_{2}(\tilde{P})}
                    \big)
        \;\leq\;
        \ln N_{[\,]}\big(2\varepsilon\|g^{\prime}\|_{L_{2}(\tilde{P})},
                         \G^{\prime},\|\cdot\|_{L_{2}(\tilde{P})}
                    \big)
        \;\leq } \\
     &\leq& \ln N\big(\varepsilon,\F_{c_{0}},\|\cdot\|_{\infty}
                      \big)
        \;\stackrel{(\ref{lemma-donsker-p2})}{\leq}\;
        r\cdot
             a_{c_{0}}^{\frac{d}{m}}
             \left(\frac{1}{\varepsilon}\right)^{\frac{d}{m}}\;.
             \qquad\qquad\qquad\qquad\qquad
   \end{eqnarray*}
   Hence, the assumption $m>\!\frac{d}{2}$ implies 
   that $\G^{\prime}$ has a finite uniform entropy integral.
   
 \textit{Part 3:} Now, it will be shown that
   $\G$ is a $P$\,-\,Donsker class.
   Trivially, $\{b\}$ is a $P$\,-\,Donsker class because
   $b\in L_{2}(P)$ by assumption.
   From 
   \cite[Example 2.5.4]{vandervaartwellner1996}
   it follows that $\G_{1}$ is $P$\,-\,Donsker.
   Note that $\,\G_{2}\,=\,\G^{\prime}\cdot\F_{c}\,$ for
   $c=1$\,. According to Part 1, the class
   $\F_{c}$ has a finite uniform entropy integral
   relative to the (constant) envelope $f_{c}$ and,
   according to Part 2, the class $\G^{\prime}$
   has a finite uniform entropy integral relative
   to the envelope $g^{\prime}$\,. Therefore, it follows from
   \cite[Example 19.19]{vandervaart1998}
   that $\,\G_{2}\,=\,\G^{\prime}\cdot\F_{c}\,$ has a
   finite uniform entropy integral relative
   to the envelope $f_{c}g^{\prime}$\,.
   The definitions and assumptions imply 
   $\,\int(f_{c}g^{\prime})^2 \,dP\,<\,\infty\,$.\\
   Hence, it follows from
   \cite[Theorem 19.4]{vandervaart1998}
   that $\G_{2}$ is a $P$\,-\,Donsker class
   provided that $\G_{2}$ is ``suitably measurable''.
   According to
   \cite[p.\ 274]{vandervaart1998},
   it suffices to show that there is a countable
   subset $\hat{\G}_{2}\subset\G_{2}$ such that,
   for every $g\in\G_{2}$\,, there is a sequence
   $(\hat{g}_{n})_{n\in\N}\subset\hat{\G}_{2}$ which converges
   pointwise to $g$\,. 
   According to \cite[Lemma 4.33]{steinwart2008}, 
   $H$ is a separable Hilbert space
   and, therefore, the subsets $\F_{c}\subset H$ are
   also separable for $c=1$ and $c=c_{0}$\,. That is,
   there are countable subsets 
   $\hat{\F}_{1}\subset\F_{1}$ and $\hat{\F}_{c_{0}}\subset\F_{c_{0}}$
   which are dense in $\F_{1}$ and $\F_{c_{0}}$ respectively
   (with respect to the norm topology).
   Then,
   $$\hat{\G}_{2}\;:=\;
     \big\{\Ls_{\hat{f}_{0}}\hat{f}_{1}\;
     \big|\;\;\hat{f}_{0}\in\F_{c_{0}}\,,\;\;
              \hat{f}_{1}\in\F_{1}
     \big\}
   $$
   is again countable. Fix any $g\in\G_{2}$\,. That is,
   there are $f_{0}\in\F_{c_{0}}$ and $f_{1}\in\F_{1}$
   such that $g=\Ls_{f_{0}}f_{1}$\,. 
   Furthermore, there are sequences 
   $\big(\hat{f}_{0}^{(n)}\big)_{n\in\N}\in\F_{c_{0}}$
   and
   $\big(\hat{f}_{1}^{(n)}\big)_{n\in\N}\in\F_{1}$
   such that
   $$\lim_{n\rightarrow\infty}
       \big\|\hat{f}_{0}^{(n)}-f_{0}\|_{H}\;=\;0
     \qquad\text{and}\qquad
     \lim_{n\rightarrow\infty}
       \big\|\hat{f}_{1}^{(n)}-f_{1}\|_{H}\;=\;0\;.
   $$
   Next, define 
   $\hat{g}_{n}:=\Ls_{\hat{f}_{0}^{(n)}}\hat{f}_{1}^{(n)}
     \in\hat{\G}_{2}
   $ for every $n\in\N$.
   Since $H$ is a reproducing kernel Hilbert space,
   norm convergence implies pointwise convergence so that,
   for every $(x,y)\in\XY$,
   $$\;\lim_{n\rightarrow\infty}\hat{g}_{n}(x,y)
     =\lim_{n\rightarrow\infty}
          \Ls\big(x,y,\hat{f}_{0}^{(n)}\!(x)\big)\hat{f}_{1}^{(n)}\!(x)
        =\Ls\big(x,y,f_{0}(x)\big)f_{1}(x)=
     g(x,y)
   $$
   due to continuity of $\Ls$\,.
    
 \textit{Part 4:} As
   $\G$ is assured to be a $P$\,-\,Donsker class,
   we have
   $$\sqrt{n}\big(\mathds{F}_{\mathbf{D}_{n}}-\iota^{-1}(P)\big)
    \;\leadsto\;
    \mathds{G}
    \qquad\text{in}\quad\ell_{\infty}(\G)
   $$
   where $\mathds{G}:\Omega\rightarrow\ell_{\infty}(\G)$
   is a tight Borel-measurable Gaussian process.
   Since 
   $\,\sqrt{n}\big(\mathds{F}_{\mathbf{D}_{n}(\omega)}-\iota^{-1}(P)\big)
    \,\in\,B_{0}\,
   $
   for every $\omega\in\Omega$
   and every $n\in\N$, it follows from closedness of
   $B_{0}$ and the Portmanteau theorem
   \cite[Theorem 1.3.4(iii)]{vandervaartwellner1996} 
   that $\mathds{G}(\omega)\in B_{0}$ almost surely. Hence, we may assume
   without loss of generality that
   $\mathds{G}(\omega)\in B_{0}$ for every $\omega\in\Omega$.
   (Otherwise, replace $\mathds{G}$ by 
   $\mathds{G}\cdot (I_{B_{0}}\circ\mathds{G})$\,.) 
\end{proof}

For ease of reference, the following lemma 
summarizes some facts about
Bochner-integrals of tight Gaussian processes
in a space $\ell_{\infty}(T)$. Later on, these facts are needed in order to
prove that the Gaussian process $\mathds{H}:\Omega\rightarrow H$
is zero-mean.
\begin{lemma}\label{lemma-mean-of-gaussian-process}
  Let $T$ be any set, $\ell_{\infty}(T)$
  the set of all bounded functions $h:T\rightarrow\R$
  (endowed with the supremum-norm)
  and $\mathds{G}:\Omega\rightarrow\ell_{\infty}(T)$
  a tight Borel-measurable Gaussian process such that
  \begin{eqnarray}\label{lemma-mean-of-gaussian-process-1}
    \int \mathds{G}(\omega)(t)\,Q(d\omega)\;=\;0
    \qquad\quad\forall\,t\in T\;.
  \end{eqnarray}
  Then, the Bochner-integral of
  $\mathds{G}:\Omega\rightarrow\ell_{\infty}(T)$
  exists and
  $\int \mathds{G}(\omega)\,Q(d\omega)=0$.
  Furthermore,
  $\int\! A(\mathds{G})\,dQ = 0$
  for every Banach space $E$ and every continuous linear
  operator $A:\ell_{\infty}(T)\rightarrow E$\,.
\end{lemma}
\begin{proof}\item
  Since $\mathds{G}$ is tight, it is also 
  separable so that there is a separable
  subset $\Gamma\subset\ell_{\infty}(T)$ such that
  $Q(\mathds{G}\in\Gamma)=1$\,; see 
  \cite[16f]{vandervaartwellner1996}.
  As the closed linear span of a separable subset
  of a Banach space is again separable
  \cite[Lemma A.48]{schechter2004},
  we may assume without loss of generality that
  $\Gamma$ is a separable Banach space. Define
  $\hat{\mathds{G}}=\mathds{G}\cdot(I_{\Gamma}\circ\mathds{G})$\,. Then,
  $\hat{\mathds{G}}:\Omega\rightarrow\Gamma$ is 
  a Borel-measurable map. Let 
  $\hat{h}^{\ast}:\Gamma\rightarrow\R$ be 
  a continuous linear functional. According to
  the Hahn-Banach-Theorem 
  \cite[Theorem II.3.11]{dunford1958},
  $\hat{h}^{\ast}$ can be extended to a
  continuous linear functional
  $h^{\ast}:\ell_{\infty}(T)\rightarrow\R$\,. 
  Since $h^{\ast}(\mathds{G})$ is normally distributed
  according to 
  \cite[Lemma 3.9.8]{vandervaartwellner1996} and
  $\hat{h}^{\ast}(\hat{\mathds{G}})=h^{\ast}(\mathds{G})\;\,
   Q-\text{a.s.}\,,
  $
  the real random variable $\hat{h}^{\ast}(\hat{\mathds{G}})$
  is normally distributed. This proves that the Borel-measurable map
  $\hat{\mathds{G}}:\Omega\rightarrow\Gamma$ is a 
  Gaussian process in the separable Banach space $\Gamma$\,.
  Hence, it follows from \cite{sato1971}
  that
  $\int \|\hat{\mathds{G}}\|\,dQ<\infty$ and, therefore,
  \begin{eqnarray}\label{lemma-mean-of-gaussian-process-p1}
    \int \|\mathds{G}\|\,dQ\;<\;\infty\;. 
  \end{eqnarray}
  (\cite{fernique1970} proves a related statement for centered
  Gaussian processes but we still have to
  prove that $\mathds{G}$ is centered and this will be done
  by use of (\ref{lemma-mean-of-gaussian-process-p1})
  so that we cannot use Fernique's theorem here.)
  According to 
  \cite[Theorem 3.10.3 and Theorem 3.10.9]{denkowski2003}, 
  (\ref{lemma-mean-of-gaussian-process-p1})
  is equivalent to the existence of the Bochner-integral
  $\int \mathds{G}\,dQ$\,.

  Note that, for every $t\in T$, the map
  $\tau_{t}:\ell_{\infty}(T)\rightarrow\R,\;h\mapsto h(t)$
  is a continuous linear operator. Then, by use of
  the fact that the Bochner-integral may be interchanged with
  continuous linear operators 
  \cite[Theorem 3.10.16 and Remark 3.10.17]{denkowski2003}, we get
  \begin{eqnarray*}
    \left(\int \mathds{G}(\omega)\,Q(d\omega)
      \right)\!(t)
    &=&\tau_{t}\left(\int \mathds{G}(\omega)\,Q(d\omega)\right)
       \;=\;\int \tau_{t}\big(\mathds{G}(\omega)\big)\,Q(d\omega)\;=\\
    &=&\int \mathds{G}(\omega)(t)\,Q(d\omega)
       \;\stackrel{(\ref{lemma-mean-of-gaussian-process-1})}{=}\;0
  \end{eqnarray*}
  for every $t\in T$\,. That is, $\int \mathds{G}\,dQ=0$.
  Using again the fact that the Bochner-integral may be interchanged with
  continuous linear operators, we finally get
  $\int\! A(\mathds{G})\,dQ
    =A\left(\int \mathds{G} \,dQ
          \right)
    =A(0)=0
  $.
\end{proof}

\begin{proof}
  \item[\textbf{Proof of Theorem \ref{theorem-sqrt-n-consistency}:}]
  First, it will be shown that
  $$\Omega\;\rightarrow\;H\;,\qquad
    \omega\;\mapsto\;
    f_{L,\mathbf{D}_{n}(\omega),\lambda_{\mathbf{D}_{n}(\omega)}}
  $$
  is Borel-measurable. According to the assumptions,
  it follows from \cite[Lemma 5.13 and Corollary 5.19]{steinwart2008}
  that
  $(\XY)^n\rightarrow H,\;\;D_n\mapsto f_{L,D_n,\lambda}$
  is continuous for every constant $\lambda\in(0,\infty)$ and
  that
  $(0,\infty)\rightarrow H,\;\;\lambda\mapsto f_{L,D_n,\lambda}$
  is continuous for every $D_n\in(\XY)^n$. Hence,
  $(D_n,\lambda)\mapsto f_{L,D_n,\lambda}$ is a
  Carath\'{e}odory function and, therefore, measurable;
  see, e.g., \cite[Theorem 2.5.22]{denkowski2003}. Since
  $\omega\mapsto\mathbf{D}_{n}(\omega)$
  and $\omega\mapsto\lambda_{\mathbf{D}_{n}(\omega)}$
  are assumed to be measurable, the compound function
  $\omega\mapsto
    f_{L,\mathbf{D}_{n}(\omega),\lambda_{\mathbf{D}_{n}(\omega)}}
  $ 
  is again measurable.

  In order to apply the functional delta-method
  \cite[Theorem 3.9.4]{vandervaartwellner1996},
  note that 
  $\ell_{\infty}(\G)$ and $H$ are Banach spaces.  
  Recall from Lemma \ref{lemma-donsker} that 
  $\mathds{F}_{\mathbf{D}_{n}}:\Omega\rightarrow B_S,\;\;
   \omega\mapsto\mathds{F}_{\mathbf{D}_{n}(\omega)}
  $ 
  is the random map where $\mathds{F}_{\mathbf{D}_{n}(\omega)}$
  is that element of $B_S$ which corresponds to the empirical
  distribution of 
  $\mathbf{D}_{n}(\omega)=\big((X_1(\omega),Y_1(\omega)),\dots,
   (X_n(\omega),Y_n(\omega))\big)
  $. That is, 
  $$\mathds{F}_{\mathbf{D}_{n}(\omega)}\;:\;\;
    \G\;\rightarrow\;\R\;,\qquad
    g\;\mapsto\;\frac{1}{n}\sum_{i=1}^{n}g\big(X_i(\omega),Y_i(\omega)\big)\;.
  $$
  Define
  $$F_{0}:=\iota^{-1}(P)
    \qquad\text{and}\qquad
    \xi_{n}
    :=\frac{\lambda_{0}}{\lambda_{\mathbf{D}_n}}\mathds{F}_{\mathbf{D}_{n}}\;.
  $$
  Then, 
  Lemma \ref{lemma-donsker} yields
  $$\sqrt{n}\big(\mathds{F}_{\mathbf{D}_{n}}-F_{0}\big)
    \;\;\leadsto\;\;
    \mathds{G}
    \qquad\text{in}\quad\ell_{\infty}(\G)
  $$
  where $\mathds{G}:\Omega\rightarrow\ell_{\infty}(\G)$
  is a tight Borel-measurable Gaussian process
  which takes it values in $B_{0}$\,.
  Furthermore, 
  \begin{eqnarray}\label{theorem-sqrt-n-consistency-p201}
    \int\mathds{G}(\omega)(g)\,Q(d\omega)\;=\;0
    \qquad\quad\forall\,g\in\G\;;
  \end{eqnarray}
  see \cite[p.\ 81f]{vandervaartwellner1996}.
  According to \cite[p.\ 16f]{vandervaartwellner1996},
  $\mathds{G}$ is also separable
  (which is important
  in order to apply Slutsky's lemma for
  Banach space valued random maps below).
  Note that
  $\sqrt{n}\big(\lambda_{\mathbf{D}_n}-\lambda_0\big)\rightarrow0$
  in probability implies
  $\lambda_0/\lambda_{\mathbf{D}_n}\rightarrow 1$ and
  $\sqrt{n}
   \big(\lambda_{\mathbf{D}_n}-\lambda_0\big)/\lambda_{\mathbf{D}_n}
   \rightarrow0
  $
  in probability; see e.g.\
  \cite[Theorems 2.3 and 2.7vi]{vandervaart1998}.
  Hence, it follows from Slutsky's lemma
  \cite[p.\ 32]{vandervaartwellner1996} that
  $$\sqrt{n}\big(\xi_{n}-F_{0}\big)\,\,=\,\,
    \sqrt{n}\big(\mathds{F}_{\mathbf{D}_{n}}-F_{0}\big)
       \cdot\frac{\lambda_{0}}{\lambda_{\mathbf{D}_{n}}}
       \,+\,
   \frac{\sqrt{n}
         \big(\lambda_{\mathbf{D}_n}-\lambda_0\big)
        }{\lambda_{\mathbf{D}_n}}
    \;\;\leadsto\;\;
    \mathds{G}
  $$
  in $\ell_{\infty}(\G)$.
  Then, applying the delta-method 
  \cite[Theorem 3.9.4]{vandervaartwellner1996}
  yields
  \begin{eqnarray*}
    \sqrt{n}\big(f_{L,\mathbf{D}_{n},\lambda_{\mathbf{D}_n}}-\SVMlplz\big)
    \;\stackrel{(\ref{prep-standard-regularization-parameter})}{=}\;
    \sqrt{n}\big(S(\xi_{n})-S(F_{0})\big)
    \;\leadsto\;S_{F_{0}}^{\prime}(\mathds{G})\;.
  \end{eqnarray*}
  Since $S^{\prime}_{F_{0}}$ is a continuous linear operator and
  $\mathds{G}$ is a tight Borel-measurable Gaussian process,
  $S_{F_{0}}^{\prime}(\mathds{G})$ is Gaussian as well;
  see, e.g., \cite[\S\,3.9.2]{vandervaartwellner1996}.
  Since $H$ is a complete and separable metric space,
  $S_{F_{0}}^{\prime}(\mathds{G})$ is tight;
  see e.g.\ \cite[Theorem 11.5.4]{dudley2002}.
  
  It follows from (\ref{theorem-sqrt-n-consistency-p201}) and 
  Lemma \ref{lemma-mean-of-gaussian-process}
  that $S_{F_{0}}^{\prime}(\mathds{G})$ has mean zero.
\end{proof}

\begin{proof}
  \item[\textbf{Proof of Theorem \ref{theorem-sqrt-n-consistency-risks}:}]
  It follows from
  Lemma \ref{lemma-derivatives} 
  that the risk functional
  $\mathcal{R}_{L,P}
  $
  is Hadamard-differentiable 
  in $H$ tangentially to $H$; the derivative 
  of $\mathcal{R}_{L,P}$ in $f\in H$ is 
  the continuous linear operator 
  $$\mathcal{R}_{L,P;f}^{\prime}\;:\;\;H\;\rightarrow\;\R\,,\qquad
    h\;\mapsto\;\Big\la\int L_{f}^{\prime}\Phi \,dP\,,\,h \Big\ra\;.
  $$
  According to Theorem \ref{theorem-sqrt-n-consistency},
  $\sqrt{n}\big(f_{L,\mathbf{D}_{n},\lambda_{\mathbf{D}_{n}}}-\SVMlplz\big)
   \leadsto\mathds{H}
  $
  where $\mathds{H}:\Omega\rightarrow H$
  is a tight Borel-measurable Gaussian process
  which has zero-mean and does not depend on 
  $\lambda_{\mathbf{D}_{n}}$ but only on $\lambda_0$.
  Then, it follows from 
  the delta-method 
  \cite[Theorem 3.9.4]{vandervaartwellner1996}
  that
  $$\sqrt{n}
    \big(\mathcal{R}_{L,P}(f_{L,\mathbf{D}_{n},\lambda_{\mathbf{D}_{n}}})
                 -\mathcal{R}_{L,P}(\SVMlplz)
    \big)
    \;\;\leadsto\;\;\mathcal{R}_{L,P;\SVMlplz}^{\prime}(\mathds{H})\;.
  $$
  Since $\mathcal{R}_{L,P;\SVMlplz}^{\prime}$ is a continuous linear operator,
  and $\mathds{H}$ is Gaussian, the 
  (real valued) random variable
  $\mathcal{R}_{L,P;\SVMlplz}^{\prime}(\mathds{H})$ is normally distributed;
  see e.g.\ \cite[\S\,3.9.2]{vandervaartwellner1996}.
  Therefore, it only remains to prove that
  the mean of $\mathcal{R}_{L,P;\SVMlplz}^{\prime}(\mathds{H})$ 
  is equal to 0.
  This follows from 
  $$\mathds{E}\mathcal{R}_{L,P;\SVMlplz}^{\prime}(\mathds{H})
    \;=\;\mathds{E}
          \Big\la\int L_{\SVMlplz}^{\prime}\Phi\,dP\,,\,\mathds{H}\Big\ra
    \;=\;0
  $$
  as $\mathds{H}:\Omega\rightarrow H$
  has zero-mean.
\end{proof}

\begin{proposition}\label{prop-degenerated-limit}
  Under the assumptions of Theorem 
  \ref{theorem-sqrt-n-consistency},
  the Gaussian process 
  $$\mathds{H}\;:\;\;\Omega\;\rightarrow\;H\,,\quad\;
    \omega\;\mapsto\;\mathds{H}(\omega)
  $$ 
  in (\ref{theorem-sqrt-n-consistency-2}) is degenerated
  to 0 if and only if
  for every $h\in H$, there is 
  a constant $c_h\in\R$ such that
  \begin{eqnarray}\label{prop-degenerated-limit-1}
    L^\prime\big(x,y,\SVMlplz(x)\big)h(x)
    \;=\;c_h\quad \;
    \text{for}\quad
    P\,-\,\textup{a.e. }\,
    (x,y)\in\XY\;.\quad  
  \end{eqnarray}
\end{proposition}
\begin{proof}\item
  According to the proof of Theorem
  \ref{theorem-sqrt-n-consistency},
  the Gaussian process $\mathds{H}$ is equal to
  $S_{F_0}^{\prime}(\mathds{G})$ and, according to 
  (\ref{lemma-prep-hadamard-p3001}), 
  $S_{F_0}^{\prime}(\mathds{G})$ is equal to 0
  if and only if $\mathds{G}\big(\Ls_{\SVMifz} h\big)$ is equal to 0
  for every $h\in H$ such that $\|h\|_H\leq 1$. As shown
  in Lemma \ref{lemma-donsker}, the class of functions $\G$ is a
  P-Donsker class and, accordingly, the distribution of the
  marginals
  $\mathds{G}\big(\Ls_{\SVMifz} h\big)$ of the limit of
  $\sqrt{n}\big(\mathds{F}_{\mathbf{D}_{n}}-\iota^{-1}(P)\big)
   \,\,\leadsto\,\,
   \mathds{G}
  $ 
  in $\ell_\infty(\G)$ is equal to
  $\mathcal{N}(0,\sigma^2_h)$ where
  $$\sigma^2_h
    =\int \bigg( \Ls_{\SVMifz} h -\int \Ls_{\SVMifz} h\,\,dP
         \bigg)^2
    \,dP\;;
  $$
  see e.g.\ \cite[\S\,2.1]{vandervaartwellner1996}. That is,
  $\mathds{H}=0$ almost surely if and only if
  $\sigma^2_h=0$ for every $h\in H$. 
\end{proof}

\bibliographystyle{abbrvnat}
\bibliography{literatur}

\end{document}